\documentclass[accepted]{article}
\usepackage{microtype}
\usepackage{graphicx} 
\usepackage{subfigure}
\usepackage{booktabs} 
\usepackage{commath}
\usepackage{footnote}
\usepackage{lipsum}
\usepackage{etoolbox}

\BeforeBeginEnvironment{figure}{\vskip-10ex}
\AfterEndEnvironment{figure}{\vskip-10ex}
\usepackage{hyperref}

\usepackage{multicol, blindtext}
\usepackage{icml2019}
\usepackage{amsmath,amsfonts,bm}
\usepackage{nicefrac}       
\usepackage{microtype}      
\usepackage{graphics}
\usepackage[nottoc]{tocbibind}
\usepackage{textcomp}
\usepackage{threeparttable}
\usepackage[marginal, flushmargin]{footmisc}
\usepackage{rotating} 
\usepackage{comment}
\usepackage{etoolbox}
\usepackage{algorithmic}
\usepackage{float}
\newbool{inccomment}
\setbool{inccomment}{true}
\renewcommand{\algorithmicrequire}{\textbf{Initialization:}}
\renewcommand{\algorithmicensure}{\textbf{Iteration:}} 
\usepackage{mathrsfs}
\usepackage{epsfig}
\usepackage{amsmath}
\usepackage{lipsum}
\usepackage{ntheorem}
\usepackage{amsfonts, amssymb}
\usepackage{chemarrow}
\usepackage{epsfig}
\usepackage{hyperref}
\usepackage[T1]{fontenc}
\usepackage{stmaryrd}
\usepackage{dcolumn}
\usepackage{graphicx}
\usepackage{amssymb}
\usepackage{dcolumn}
\usepackage{bm}
\usepackage{color}
\usepackage{epstopdf}
\usepackage{algorithm}
\makeatletter  
\def\BState{\State\hskip-\ALG@thistlm}
\makeatother
\usepackage{nomencl}
\usepackage{multirow}
\usepackage{fancyhdr}
\usepackage{datetime}
\usepackage{graphicx}  
\usepackage{url}
\usepackage{hyperref}
\usepackage{empheq}
\usepackage{mathtools}
\usepackage{arydshln}

\def\eqref#1{equation~\ref{#1}}

\def\1{\bm{1}}

\def\rmH{{\mathbf{H}}}
\def\rmI{{\mathbf{I}}}

\DeclareMathAlphabet{\mathsfit}{\encodingdefault}{\sfdefault}{m}{sl}
\SetMathAlphabet{\mathsfit}{bold}{\encodingdefault}{\sfdefault}{bx}{n}

\newcommand{\R}{\mathbb{R}}

\newcommand{\bw}{\boldsymbol{W}}
\newcommand{\by}{\mathbf{y}}
\newcommand{\Hessian}{\mathbf{H}}
\newcommand{\grad}{\mathbf{g}}

\newcommand{\Cov}{\boldsymbol{\Pi}}

\newcommand{\m}{M}
\newcommand{\nblock}{o}
\newcommand{\Nblock}{O}
\newcommand{\dimBr}{\aleph}

\newcommand{\bx}{\mathbf{x}}
\newcommand{\bD}{\mathbf{D}}
\newcommand{\bN}{\mathcal{N}}

\newcommand{\btheta}{\boldsymbol{}}
\newcommand{\bW}{\mathbf{W}}
\newcommand{\define}{\triangleq}
\newcommand{\diag}{\operatorname{diag}}
\newcommand{\mean}{\mathbf{m}_{jk}^{o'}}
\newcommand{\variance}{C_{jk}^{o'}}
\newcommand{\given}{\,|\,}
\newcommand{\DistGam}{\text{Gam}}
\newlength{\bracewidth}
\newcommand{\myunderbrace}[2]{\settowidth{\bracewidth}{$#1$}#1\hspace*{-1\bracewidth}\smash{\underbrace{\makebox{\phantom{$#1$}}}_{#2}}}

\newcommand{\quadratic}{{\frac{1}{2}(\bw-\bw^*)^{\top} \Hessian(\bw^* \btheta) (\bw-\bw^*)+  (\bw-\bw^*)^{\top}\grad(\bw^* \btheta)}}

\DeclareMathOperator*{\argmin}{arg\,min}

\newtheorem{proposition}{Proposition}
\newtheorem{remark}{Remark}
\newtheorem{corollary}{Corollary}

\newtheorem*{proof}{Proof}
\newcommand{\SKIP}[1]{}

\icmltitlerunning{BayesNAS}
\begin{document}
\twocolumn[
\icmltitle{BayesNAS: A Bayesian Approach for Neural Architecture Search}



\icmlsetsymbol{equal}{*}

\begin{icmlauthorlist}
\icmlauthor{Hongpeng Zhou}{delft,equal}
\icmlauthor{Minghao Yang}{delft,equal}
\icmlauthor{Jun Wang}{ucl}
\icmlauthor{Wei Pan}{delft}
\end{icmlauthorlist}

\icmlaffiliation{delft}{Department of Cognitive Robotics, Delft University of Technology, Netherlands}
\icmlaffiliation{ucl}{Department of Computer Science, University College London, UK}
\icmlcorrespondingauthor{Wei Pan}{wei.pan@tudelft.nl}

\icmlkeywords{Machine Learning, ICML}
\vskip 0.3in
]

\printAffiliationsAndNotice{\hskip 0.2in \icmlEqualContribution} 

\begin{abstract}
One-Shot Neural Architecture Search (NAS) is a promising method to significantly reduce search time without any separate training. It can be treated as a Network Compression problem on the architecture parameters from an over-parameterized network. However, there are two issues associated with most one-shot NAS methods. First, dependencies between a node and its predecessors and successors are often disregarded which result in improper treatment over \emph{zero} operations. Second, architecture parameters pruning based on their magnitude is questionable. In this paper, we employ the classic Bayesian learning approach to alleviate these two issues by modeling architecture parameters using \emph{hierarchical automatic relevance determination} (HARD) priors. Unlike other NAS methods, we train the over-parameterized network for only \emph{one} epoch then update the architecture. Impressively, this enabled us to find the architecture on CIFAR-10 within only $0.2$ GPU days using a single GPU. Competitive performance can be also achieved by transferring to ImageNet. As a byproduct, our approach can be applied directly to compress convolutional neural networks by enforcing structural sparsity which achieves extremely sparse networks without accuracy deterioration.
\end{abstract}

\section{Introduction}

Neural Architecture Search (NAS), the process of automating architecture engineering, is thus a logical next step in automating machine learning since \citep{nas}. There are basically three existing frameworks for neural architecture search. Reinforcement learning based NAS \citep{baker2016designing,nas,zhong2018practical, zoph2017learning, cai2018path} methods take the generation of a neural architecture as an agent's action with the action space identical to the search space. More recent neuro-evolutionary approaches \citep{real2017large,liu2017hierarchical,real2018regularized,miikkulainen2019evolving, xie2017genetic, elsken2018efficient} use gradient-based methods for optimizing weights and solely use evolutionary algorithms for optimizing the neural architecture itself. However, these two frameworks take enormous computational power when compared to a search using a single GPU. One-Shot based NAS is a promising approach to significantly reduce search time without any separate training, which treats all architectures as different subgraphs of a supergraph (the one-shot model) and shares weights between architectures that have edges of this super-graph in common \citep{saxena2016convolutional, brock2017smash,pham2018efficient, bender2018understanding,liu2018darts, cai2019iclr,xie2019snas,zhang2018graph,zhang2018single}. A comprehensive survey on Neural Architecture Search can be found in \citep{survey}.

Our approach is a one-shot based NAS solution which treats NAS as a Network Compression/pruning problem on the architecture parameters from an over-parameterized network. However, despite it's remarkable less searching time compared to reinforcement learning and neuro-evolutionary approaches, we can identify a number of significant and practical disadvantages of the current one-shot based NAS. First, dependencies between a node and its predecessors and successors are disregarded in the process of identifying the redundant connections. This is mainly motivated by the improper treatment of \emph{zero} operations. On one hand, the logit of \emph{zero} may dominate some of the edges while the child network still has other non-zero edges to keep it connected \citep{liu2018darts,xie2019snas,cai2019iclr,zhang2018single}, for example, node 2 in Figure\ref{fig:1}a. Similarly, as shown in Figure 1 of \citep{xie2019snas}, the probability of invalid/disconnected graph sampled will be $\frac{511}{1024}$ when there are three non-zero plus one \emph{zero} operation. Though post-processing to safely remove isolated nodes is possible, \textit{e.g.}, for chain-like structure, it demands extensive extra computations to reconstruct the graph for complex search space with additional layer types and multiple branches and skip connections. This may prevent the use of modern network structure as the backbone such as DenseNet \citep{huang2017densely}, newly designed motifs \citep{liu2017hierarchical} and complex computer vision tasks such as semantic segmentation \citep{liu2019auto}. On the other hand, \emph{zero} operations should have higher priority to rule out other possible operations, since \emph{zero} operations equal to all \emph{non-zero} operations not being selected. Second, most one-shot NAS methods \citep{liu2018darts, cai2019iclr,xie2019snas,zhang2018single,gordon2018morphnet} rely on the magnitude of architecture parameters to prune redundant parts and this is not necessarily true. From the perspective of Network Compression \citep{lee2018snip}, magnitude-based metric depends on the scale of weights thus requiring pre-training and is very sensitive to the architectural choices. Also the magnitude does not necessarily imply the optimal edge. Unfortunately, these drawbacks exist not only in Network Compression but also in one-shot NAS.

In this work, we propose a novel, efficient and highly automated framework based on the classic Bayesian learning approach to alleviate these two issues simultaneously. We model architecture parameters by a \emph{hierarchical automatic relevance determination} (HARD) prior. The dependency can be translated by multiplication and addition of some independent Gaussian distributions. The classic Bayesian learning framework \cite{mackay1992bayesian, neal1995bayesian, tipping2001sparse} prevents overfitting and promotes sparsity by specifying sparse priors. The uncertainty of the parameter distribution can be used as a new metric to prune the redundant parts if its associated entropy $\frac{1}{2}\ln(2\pi e {\gamma_{jk}^{o'}})$ is nonpositive. The majority of parameters are automatically zeroed out during the learning process.

\paragraph{Our Contributions}
\begin{itemize}
    \item \textbf{Bayesian approach:} 
    BayesNAS is the first Bayesian approach for one-shot NAS.
    Therefore, our approach shares the advantages of Bayesian learning, which prevents overfitting and does not require tuning a lot of hyperparameters. Hierarchical sparse priors are used to model the architecture parameters. 
    Priors can not only promote sparsity, but model the dependency between a node and its predecessors and successors ensuring a connected derived graph after pruning. Furthermore, it provides a principled way to prioritize \emph{zero} operations over other \emph{non-zero} operations. In our experiment on CIFAR-10, we found that the variance of the prior, as well as that of posterior, is several magnitudes smaller than posterior mean which renders a good metric for architecture parameters pruning.
    \item \textbf{Simple and fast search:} 
    Our algorithm is formulated simply as an iteratively re-weighted $\ell_1$ type algorithm \citep{candes2008enhancing} where the re-weighting coefficients used for the next iteration are computed not only from the value of the current solution but also from its posterior variance. The update of posterior variance is based on Laplace approximation in Bayesian learning which requires computation of the inverse Hessian of log likelihood. To make the computation for large networks feasible, a fast Hessian calculation method is proposed. In our experiment, we train the model for only \emph{one} epoch before calculating the Hessian to update the posterior variance. Therefore, the search time for very deep neural networks can be kept within $0.2$ GPU days.
    \item \textbf{Network compression:} As a byproduct, our approach can be extended directly to Network Compression by enforcing various structural sparsity over network parameters. Extremely sparse models can be obtained at the cost of minimal or no loss in accuracy across all tested architectures. This can be effortlessly integrated into BayesNAS to find sparse architecture along with sparse kernels for resource-limited hardware.
\end{itemize}
\begin{figure*}
    \centering 
    \includegraphics[scale = 0.22]{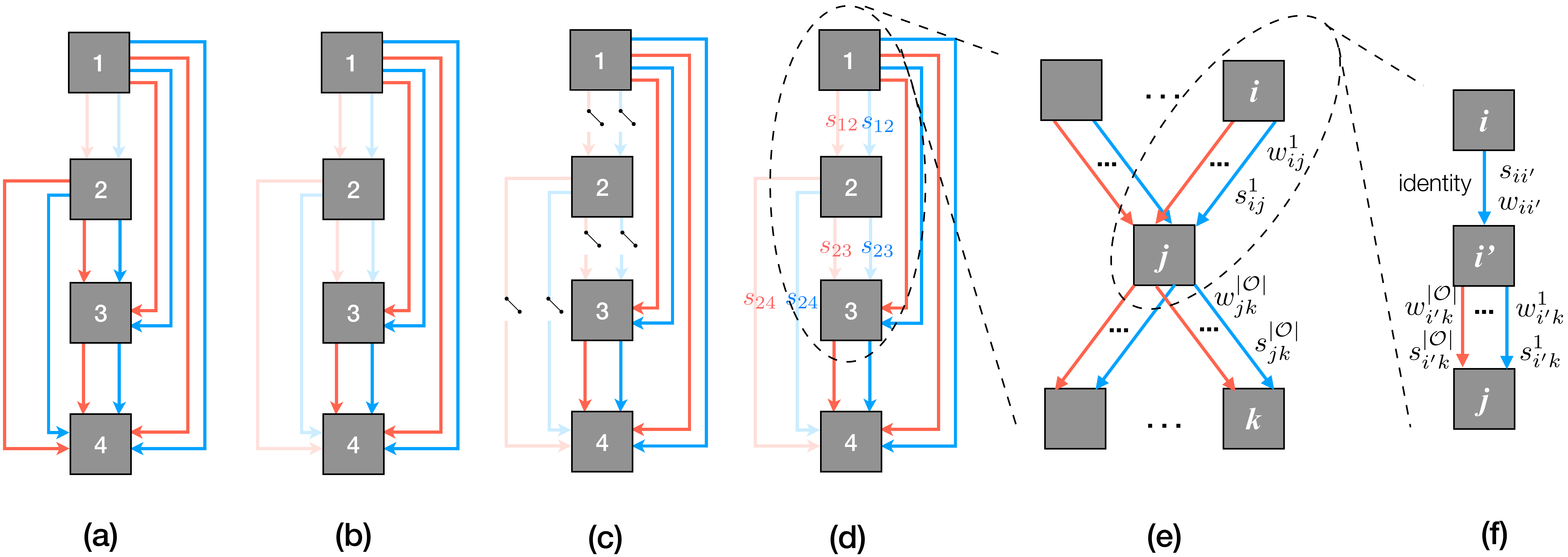}
    \vspace{-0.3cm}
    \caption{An illustration of BayesNAS:
    (a) disconnected graph with isolated node 2 caused by disregard for dependency; 
    (b) expected connected graph with no connection from node 2 to 3 and from node 2 to 4;
    (c) illustration about dependency with predecessor's ($e_{12}$) superior control over its successors ($e_{23}$ and $e_{24} $)
    (d) designed switches realizing the dependency and determining "on or off" of the edge;
    (e) elementary multi-input-multi-output motif for a graph;
    (f) prioritized \emph{zero} operation over other \emph{non-zero} operations.
    }
    \vspace{-0.5cm}
    \label{fig:1}
\end{figure*}
\vspace{-0.3cm}
\section{Related Work}
\textbf{Network Compression.} The de facto standard criteria to prune redundant weights depends on their magnitude and is designed to be incorporated with the learning process. These methods are prohibitively slow as they require many iterations of pruning and learning steps. One category is based on the magnitude of weights. The conventional approach to achieve sparsity is by enforcing penalty terms \citep{chauvin1989back, weigend1991generalization, ishikawa1996structural}. Weights below a certain threshold could be removed. In recent years, impressive results have been achieved using the magnitude of weight as the criterion \citep{han2015deep} as well as other variations \citep{guo2016dynamic}. The other category is based on the magnitude of Hessian of loss with respect to weights, \textit{i.e.}, higher the value of Hessian, greater the importance of the parameters \citep{lecun1990optimal,hassibi1993optimal}. Despite being popular, both of these categories require pretraining and are very sensitive to architectural choices. For instance, different normalization layers affect the magnitude of weights in different ways. This issue has been elaborated in \citep{lee2018snip} where the gradient information at the beginning of training is utilized for ranking the relative importance of weights' contribution to the training loss. 

\textbf{One-shot Neural Architecture Search.}
In one-shot NAS, redundant architecture parameters are pruned based on the magnitude of weights similar to that used in Network Compression. In DARTS, \citet{liu2018darts} applied a softmax function to the magnitude of $w$ to rank the relative importance for each operation. Similar to DARTS, there are two related works: ProxylessNAS~\citep{cai2019iclr} and SNAS~\citep{xie2019snas}. ProxylessNAS binarizes $w$ using $\texttt{clip}({(1+w)}/{2},0,1)$ \citep{courbariaux2015binaryconnect} where $-1$ plays the role of threshold and edge with the highest weight will be selected in the end. While SNAS applies a softened one-hot random variable to rank the architecture parameter, \citet{gordon2018morphnet} treats the scaling factor of Batch Normalization as an edge and normalization as its associated operation. \citet{zhang2018single} proposed DSO-NAS which relaxes $\ell_0$ norm by replacing it with $\ell_1$ norm and prunes the edges by a threshold, \textit{e.g.}, the learning rate is multiplied by a predefined regularization parameter to prune edges gradually over the course of training.

\textbf{Bayesian Learning and Compression.} 
Our approach is based on Bayesian learning. In principle, the Bayesian approach to learn neural networks does not have problems of tuning a large amount of hyperparameters or overfitting the training data \citep{mackay1992practical,mackay1992bayesian,neal1995bayesian,hernandez2015probabilistic}. Recently, Bayesian learning is applied to estimate layer size and network depth in NAS problem \cite{dikov2019bayesian}. By employing sparsity-inducing priors, the obtained model depends only on a subset of kernel functions for linear models \citep{tipping2001sparse} and deep neural networks where the neurons can be pruned as well as all their ingoing and outgoing weights \citep{louizos2017bayesian}. Other Bayesian methods have also been applied to network pruning \citep{ullrich2017soft, molchanov2017variational} where the former extends the soft weight-sharing to obtain a sparse and compressed network and the latter uses variational inference to learn the dropout rate that can then be used for network pruning.

\section{Search Space Design}
The search space defines which neural architectures a NAS approach might discover in principle. Designing a good search space is a challenging problem for NAS. Some works \citep{nas, zoph2017learning, pham2018efficient, cai2018path, zhang2018single, liu2018darts, cai2019iclr} have proposed that the search space could be represented by a Directed Acyclic Graph (DAG). We denote $e_{ij}$ as the edge from node $i$ to node $j$ and $o_{ij}$ stands for the operation that is associated with edge $e_{ij}$.

Similar to other one-shot based NAS approaches  \citep{bender2018understanding,zhang2018single, liu2018darts, cai2019iclr,gordon2018morphnet}, we also include (different or same) scaling scalars over all operations of all edges to control the information flow, denoted as $w_{ij}^o$ which also represent architecture parameters. The output of a mixed operation $o_{ij}, i<j$ is defined based on the outputs of its edge
\begin{align}\label{eq:output}
  o_{j}(z_i) = \sum_{o\in  \mathcal{O}} w_{ij}^o o_{ij}(z_i).
\end{align}
Then $z_j$ can be obtained as $\sum_{i<j}o_{j}(z_i)$. 

To this end, the objective is to learn a simple/sparse subgraph while maintaining/improving the accuracy of the over-parameterized DAG \citep{bender2018understanding}.
Let us formulate the search problem as an optimization problem. Given a dataset $\bD =(\mathbf{X}, \mathbf{Y}) =  \{(\bx_n, \by_n)\}_{n=1}^N$ and the desired sparsity level $\kappa$ (\textit{i.e.}, the number of non-zero edges), one-shot NAS problem can be written as an optimization problem with the following constraints:

\vspace{-0.7cm}
\begin{equation}
\begin{aligned}
\min_{\bw} L(\bW; \bD) &= \min_{\bw} \frac{1}{N} \sum_{n=1}^N \ell(\by_n, \text{Net}(\bx_n, \boldsymbol{\mathcal{W}}, \mathbf{w}))\
\\  
\text{s.t.}\quad\bw &\in \R^{m^{\text{net}}+m^{\text{edge}}},\quad \|\mathbf{w}\|_0 \le \kappa^{\text{edge}}
\label{eq:problem}
\end{aligned}
\end{equation}
where $\bW$ are split into two parts: network parameters $\boldsymbol{\mathcal{W}} = [\mathcal{W}_{ij}^o]$ and architecture parameters $\mathbf{w} = [w_{ij}^o]$ with dimension of $m^{\text{net}}$ and $m^{\text{edge}}$ respectively, and $\|\cdot\|_0$ is the standard $\ell_0$ norm. The formulation in \eqref{eq:problem} can be substantiated by incorporating \emph{zero} operations into $\mathcal{O}$ to allow removal of $w_{ij}^o$ \citep{liu2018darts, cai2019iclr} aiming to further reduce the size of cells and improve the design flexibility.

To alleviate the negative effect induced by the dependency and magnitude-based metric whose issues have been discussed in Introduction, for each $w_{ij}^o$, we introduce a switch $s_{ij}^o$ that is analogous to the one used in an electric circuit. There are four features associated with these switches. First, the ``on-off'' status is not solely determined by its magnitude. Second, dependency will be taken into account, \textit{i.e.}, the predecessor has superior control over its successors as illustrated in Figure~\ref{fig:1}c. Third, $s_{ij}^o$ is an auxiliary variable that will not be updated by gradient descent but computed directly to switch on or off the edge. Lastly, $s_{ij}^o$ should work for both proxy and proxyless scenarios and can be better embedded into existing algorithmic frameworks \cite{liu2018darts, cai2019iclr,gordon2018morphnet}. The calculation method will be introduced later in Section~\ref{sec:performance}.

Inspired by the hierarchical representation in a DAG \citep{liu2018darts,liu2017hierarchical}, we abstract a single motif as the building block of DAG, as shown in Figure~\ref{fig:1}e. Apparently, any derived motif, path, or network can be constructed by such a multi-input-multi-output motif. It shows that a successor can have multiple predecessors and each predecessor can have multiple operations over each of its successors. Since the representation is general, each directed edge can be associated with some primitive operations (\textit{e.g.}, convolution, pooling, etc.) and a node can represent output of motifs, cells, or a network.

\section{Dependency Based One-Shot Performance Estimation Strategy}
\label{sec:performance}
\subsection{Encoding the Dependency Logic}
In the following, we will formally state the criterion to identify the redundant connections in Proposition~\ref{proposition:1}. The idea can be illustrated by Figure~\ref{fig:1}b in which both the blue and red edges from node 2 to 3 and from node 2 to 4 might be non-zeros but should be removed as a consequence. To enable this, we have the following proposition.
\begin{proposition}\label{proposition:1}
There is information flow from node $j$ to $k$ under operation $o'$ as shown in Figure~\ref{fig:1}e if and only if at least one operation of at least one predecessor of node $j$ is non-zero and $w_{jk}^{o'}$ is also non-zero.
\end{proposition}
\begin{remark}
\label{remark:1}
The same expression for Proposition~\ref{proposition:1} is: 
there is \textbf{no} information flow from node $j$ to $k$ under operation $o'$ if and only if all the operation of all the predecessors of node $j$ are zeros or ${w_{jk}^{o'}}$ is zero. This explains the incompleteness of the problem \ref{eq:problem} as well as the possible phenomenon that non-zero edges become dysfunctional in Figure~\ref{fig:1}b. 
\end{remark}
\begin{remark}\label{remark:2}
The expression to encode Proposition~\ref{proposition:1} is not unique. Some examples include but not limited to, \textit{e.g.}, ${w_{jk}^{o'}\sum_{i<j} |w_{ij}^{o}|}$, $w_{jk}^{o'}\sum_{i<j} \alpha_{ij}^{o} |w_{ij}^{o}|,\forall \alpha_{ij}^{o}\in (0,1] $, $w_{jk}^{o'}\sum_{i<j}(w_{ij}^{o})^2 $. Apparently, $\ell_0$ norm of these quantities are difficult to be included in a constraint in the optimization problem formulation in \ref{eq:problem}.
\end{remark}
As can be seen in Remark~\ref{remark:2}, we will construct a probability distribution jointly over $w_{jk}^{o'}$, $w_{ij}^{o}$, $\forall i<j$ in the sequel, denoted as 
\begin{align}
\label{eq:c}
p(c(w_{jk}^{o'}, w_{ij}^{o} )), \forall i<j.
\end{align}
where $c$ is a possible expression like in Remark~\ref{remark:2} to encode Proposition~\ref{proposition:1}.

In the following, we will show how the ``switches'' $s$ can be used to implement Proposition~\ref{proposition:1}. If we assume $s$ has two states $\{\text{ON},\text{OFF}\}$, $w_{jk}^{o'}$ is redundant when $s_{jk}^{o'}$ is $\text{OFF}$ \emph{or} all $s_{ij}^{o}$ are $\text{OFF}$, $\forall i<j,o \in \mathcal{O}$. How to use $s$ to encode the redundancy of  $w_{jk}^{o'}$, \textit{i.e.}, ${w_{jk}^{o'}\sum_{i<j}|w_{ij}^{o}| = 0}$? One possible solution is

\vspace{-0.6cm}
{\small
\begin{align}
\label{eq:gamma_logic}
\bigcup_{i<j}\bigcup_{o \in \mathcal{O}} s_{ij}^{o} \cap s_{jk}^{o'} \ \ \ \ \ \text{or} \ \ \ \ \  \overline{\overline{\bigcup_{i<j}\bigcup_{o \in \mathcal{O}} s_{ij}^{o}} \cup \overline{s_{jk}^{o'}}}
\end{align}
}If $s$ is a continuous variable with $s = \infty$ for $\text{ON}$ and $0$ for $\text{OFF}$, set union and intersection can be arithmetically represented by addition and multiplication respectively. $s$ does not directly determine the magnitude of $w$ but plays the role as uncertainty or confidence for zero magnitude. 

A straightforward way to encode this logic is to assign a probability distribution, for example Gaussian distribution, over $w_{jk}^{o'}$

\vspace{-0.6cm}
{\small
\begin{align}
p(w_{jk}^{o'}) = \bN(w_{jk}^{o'}|0, s_{jk}^{o'}), \ \ \sum_{i<j}p(w_{ij}^{o}) =  \sum_{i<j}\bN(w_{ij}^{o}|0, s_{ij}^{o})
\nonumber
\end{align}
}

\vspace{-0.5cm}
Since $w_{ij}^{o}, \forall i,j, o$ are independent with each other, we construct the following distribution to express~\eqref{eq:c}:

\vspace{-0.6cm}
{\small
\begin{equation}
\begin{aligned}
\label{eq:ppath}
p(c(w_{jk}^{o'}, w_{ij}^{o} ))&\define \bN(w_{jk}^{o'}|0, s_{jk}^{o'} )\sum_{i<j}\bN(w_{ij}^{o}|0, s_{ij}^{o})  \\
&= \bN(w_{jk}^{o'}|0, s_{jk}^{o'} )\bN\left(\sum_{i<j} \frac{s_{ij}^{o}w_{ij}^{o}}{\sum_{i<j} s_{ij}^o} \vert 0, s_{ij}^{o}\right)  \\
&= \bN \left(w_{jk}^{o'}\sum_{i<j}\frac{s_{ij}^{o}w_{ij}^{o}}{\sum_{i<j} s_{ij}^o}|0, \gamma_{jk}^{o'}\right)  
\end{aligned}
\end{equation}
}where

\vspace{-1cm}
{\small
\begin{align}
\label{eq:gamma}
\gamma_{jk}^{o'} \define  \left(\frac{1}{\sum\limits_{i<j}\sum\limits_{o\in \mathcal{O}}{s}_{ij}^{o}} + \frac{1}{s_{jk}^{o'}}\right)^{-1}.
\end{align}
}

\vspace{-0.5cm}
Since $s_{ij}^o >0$ in~\eqref{eq:ppath} always holds, regardless of what $s_{ij}^o$ is, we can use the following simpler alternative to substitute~\eqref{eq:ppath} to encode Proposition~\ref{proposition:1}:

\vspace{-0.6cm}
{\small
\begin{align}
p(c(w_{jk}^{o'}, w_{ij}^{o} ))\define \bN \left({w_{jk}^{o'}\sum_{i<j} w_{ij}^{o}}|0, \gamma_{jk}^{o'}\right).
\label{eq:gamma_new}
\end{align}
}

\vspace{-0.5cm}
Interestingly, \eqref{eq:gamma_new} and~\ref{eq:gamma_logic} are equivalent. This means that we may find an algorithm that is able to find the sparse solution in a probabilistic manner. However, Gaussian distribution, in general, does not promote sparsity. Fortunately, some classic yet powerful techniques in Bayesian learning are applicable, \textit{i.e.}, sparse Bayesian learning (SBL) \citep{tipping2001sparse,pan2017bayesian} and automatic relevance determination (ARD) prior \citep{mackay1996bayesian,neal1995bayesian} in Bayesian neural networks.

\subsection{Zero Operation Ruling All}
In our paper, we do not include \textit{zero} operation as a primitive operation. Instead, between node $i$ and $j$ we compulsively add one more node $i'$ and allow only a single \emph{identity} operation (see Figure~\ref{fig:1}f). The associated weight $w_{ii'}$ is trainable and initialized to $1$ as well as its switch $s_{ii'}$. The idea is that if $s_{ii'}$ is $\text{OFF}$, all the operations from $i'$ to $j$ will be disabled as a consequence. Then $\gamma_{jk}^{o'}$ in~\eqref{eq:gamma} can be substituted by

\vspace{-0.7cm}
{\small
\begin{align}
\label{eq:gamma-2}
\gamma_{jk}^{o'} \define  \left(\frac{1}{s_{ii'}}+\frac{1}{\sum\limits_{i'<j}\sum\limits_{o\in \mathcal{O}}{s}_{i'j}^{o}} + \frac{1}{s_{jk}^{o'}}\right)^{-1}.
\end{align}
}

\vspace{-0.6cm}
\section{Bayesian Learning Search Strategy}
\subsection{Bayesian Neural Network}

The likelihood for the network weights~$\boldsymbol{\mathcal{W}}$ and the noise precision~$\sigma^{-2}$ with data~${\mathcal{D}=(\mathbf{X},\mathbf{Y})}$ is

\vspace{-0.6cm}
{\small
\begin{align}
p(\mathbf{Y}\given\boldsymbol{\mathcal{W}},\mathbf{w}, \mathbf{X},\sigma^2) &= \prod_{n=1}^N \mathcal{N}(y_n\given \text{Net}(\mathbf{x}_n;\boldsymbol{\mathcal{W}},\mathbf{w}); \sigma^2)\,.\label{eq:likelihood}
\end{align}}To complete our probabilistic model, we specify a Gaussian prior distribution for each entry in each of the weight matrices in $\boldsymbol{\mathcal{W}}$. In particular,

\vspace{-0.5cm}
{\small
\begin{align}
&p(\boldsymbol{\mathcal{W}}\given\lambda) = \prod_{i<j} \prod_{o \in \mathcal{O}} \mathcal{N}(\mathcal{W}_{ij}^{o}\given0,\lambda^{-1})\,\label{eq:prior_weights_1}\\
&p(\mathbf{w}\given \mathbf{s}) = \prod_{j<k} \prod_{o \in \mathcal{O}} \prod_{o' \in \mathcal{O}}\bN \left({w_{jk}^{o'}\sum_{i<j} w_{ij}^{o}}|0, \gamma_{jk}^{o'}\right)
\label{eq:prior_weights_2}
\end{align}
}where $\gamma_{jk}^{o'}$ is defined in \eqref{eq:gamma-2}.  
$\sigma^{-2}$,~$\lambda$ and~$\mathbf{s}$ are \emph{hyperparameters}. Importantly, there is an individual hyperparameter associated independently with every edge weight and a single one with all network weight. Follow Mackay's evidence framework \citep{mackay1992bayesian}, 'hierarchical priors' are employed on the latent variables using Gamma priors on the inverse variances. The hyper-priors for~$\sigma^{-2}$,~$\lambda$ and~$\mathbf{s}$ are chosen to be a gamma distribution \citep{berger2013statistical}, \textit{i.e.}, $p(\lambda) = \DistGam(\lambda\given a^\lambda$, $b^\lambda), p(\beta) = \DistGam(\beta\given a^{\beta},b^{\beta})$ with $\beta = \sigma^{-2}$, and $p(s_{ij}^o) = \DistGam(s_{ij}^o \given a^{s_{ij}^o},b^{s_{ij}^o})$. Essentially, the choice of Gamma priors has the effect of making the marginal distribution of the latent variable prior the non-Gaussian Student’s \emph{t} therefore promoting the sparsity \citep[Section 2 and 5.1]{tipping2001sparse}. To make these priors non-informative (\textit{i.e.}, flat), we simply fix $a$ and $b$ to zero by assuming uniform scale priors for analysis and implementation. This formulation of prior distributions is a type of \emph{hierarchically constructed automatic relevance determination} (HARD) prior which is built upon classic ARD prior \citep{neal1995bayesian, tipping2001sparse}.

The posterior distribution for the parameters~$\boldsymbol{\mathcal{W}}$,~$\gamma$ and~$\lambda$ can then be obtained by applying Bayes' rule:

\vspace{-0.6cm}{\small
\begin{align}
&p(\boldsymbol{\mathcal{W}},\mathbf{w}, \lambda, \mathbf{s}, \sigma^2 \given\mathcal{D}) \nonumber \\
=&
\frac{p(\mathbf{Y}\given \mathbf{X}, \boldsymbol{\mathcal{W}},\mathbf{w}, \lambda, \mathbf{s}, \sigma^2)
p(\boldsymbol{\mathcal{W}}\given\lambda)p(\mathbf{w}\given \mathbf{s})p(\lambda)p(\gamma)p(\sigma^2)}{p(\mathbf{Y}\given\mathbf{X})},\label{eq:exact_posterior}
\end{align}
}where~$p(\mathbf{Y}\given\mathbf{X})$ is a normalization constant. Given a new input vector~$\mathbf{x}_\star$, we can make predictions for its output~$\mathbf{y}_\star$ using the predictive distribution given by

\vspace{-0.6cm}
{\small
\begin{equation}
\begin{aligned}
\label{eq:predictive_distribution}
& p(\mathbf{y}_\star\given\mathbf{x}_\star,\mathcal{D}) \\
= &\int\!\! p(\mathbf{y}_\star|\mathbf{x}_\star, \boldsymbol{\mathcal{W}},\mathbf{w}, \lambda, \mathbf{s}, \sigma^2) 
p(\boldsymbol{\mathcal{W}},\mathbf{w}, \lambda, \mathbf{s}, \sigma^2\given \mathcal{D})\, \\
& \ \ \ \ d\sigma^2\, d\lambda\, d\mathbf{s}\, d\boldsymbol{\mathcal{W}}d\mathbf{w},
\end{aligned}
\end{equation}
}where \begin{math}{p(\mathbf{y}_\star|\mathbf{x}_\star, \boldsymbol{\mathcal{W}},\mathbf{w}, \lambda, \mathbf{s}, \sigma^2)  =
\mathcal{N}(\mathbf{y}_\star\given \text{Net}(\mathbf{x}_\star), \sigma^2)}\end{math}. However, the exact computation of  $p(\boldsymbol{\mathcal{W}},\mathbf{w}, \lambda, \mathbf{s}, \sigma^2\given \mathcal{D})$ and $p(\mathbf{y}_\star\given\mathbf{x}_\star)$ is not tractable in most cases. Therefore, in practice, we have to resort to approximate inference methods. 

It should be noted that $\lambda$ is the same for all network parameters. However, it can be different for $\boldsymbol{\mathcal{W}}$ or constructed to represent the structural sparsity for Convolutional kernels in NN aiming for Network Compression, which is related to Bayesian compression \citep{louizos2017bayesian} and structural sparsity compression \citep{wen2016learning}. We give some examples in Figure~\ref{fig:2} and more can be found in the Appendix~\ref{app:structural compression experiment} where extremely sparse networks on MNIST and CIFAR-10 can be obtained without accuracy deterioration. Since our main focus is on architecture parameters, without breaking the flow, we will fix $\lambda$ which is equivalent to the weight decay coefficient in SGD and $\sigma^2= 0.01$ that is equivalent to the regularization coefficient for network parameters. 

In case of uniform hyperpriors, we only need to maximize the term $p(\mathbf{Y} \given  \lambda, \mathbf{s}, \sigma^2)$ \citep{mackay1992bayesian,berger2013statistical}

\vspace{-0.6cm}
{\small
\begin{equation}
\begin{aligned}
\int \int p(\mathbf{Y}\given\boldsymbol{\mathcal{W}},\mathbf{w}, \mathbf{X},\sigma^2) p(\boldsymbol{\mathcal{W}}\given\lambda) p(\mathbf{w}\given \mathbf{s}) d\boldsymbol{\mathcal{W}} d\mathbf{w}.
\label{eq:mlh}
\end{aligned}
\end{equation}}We assume that the distribution of data likelihood belongs to the exponential family

\vspace{-0.5cm}
{\small
\begin{equation}
\begin{aligned}
p(\mathbf{Y}\given\boldsymbol{\mathcal{W}},\mathbf{w}, \mathbf{X},\sigma^2) \sim \exp \left(-E_D(\mathbf{Y}; \text{Net}(\mathbf{X};\boldsymbol{\mathcal{W}}, \mathbf{w}); \sigma^2)\right)
\label{eq:expolikelihood}
\end{aligned}
\end{equation}
}where $E_D(*)$ is the \emph{energy function} over data.

\subsection{Laplace Approximation and Efficient Hessian Computation}
In related Bayesian models, the quantity in~\eqref{eq:mlh} is known as the marginal likelihood and its maximization is known as the type-II maximum likelihood method \citep{berger2013statistical}. And neural networks can also be treated in a Bayesian manner known as Bayesian learning for neural networks \citep{mackay1992practical,neal1995bayesian}. Several approaches have been proposed based on, \textit{e.g.}, the Laplace approximation \citep{mackay1992practical}, Hamiltonian Monte Carlo \citep{neal1995bayesian}, expectation propagation \citep{jylanki2014expectation, hernandez2015probabilistic}, and variational inference \citep{hinton1993keeping,graves2011practical}. Among these methods, we adopt Laplace approximation. However, Laplace approximation requires computation of the inverse Hessian of log-likelihood, which can be infeasible to compute for large networks. Nevertheless, we are motivated by 1) its easy implementation, especially using recent popular deep learning open source software; 2) versatility for modern NN structures such as CNN and RNN as well as their modern variations; 3) close relationship between computation of Hessian and Network Compression using Hessian metric  \citep{lecun1990optimal,hassibi1993optimal}; 4) acceleration effect to training convergence by second-order optimization algorithm  \citep{botev2017practical} to which it is related. 
In this paper, we propose the efficient calculation/approximation of Hessian for convolutional layer and architecture parameter. The detailed calculation procedures are explained in Appendix ~\ref{appendix:convhessian} and ~\ref{appendix:scalarhessian} respectively.

\begin{figure}[t]
    \centering
    \includegraphics[scale = 0.17]{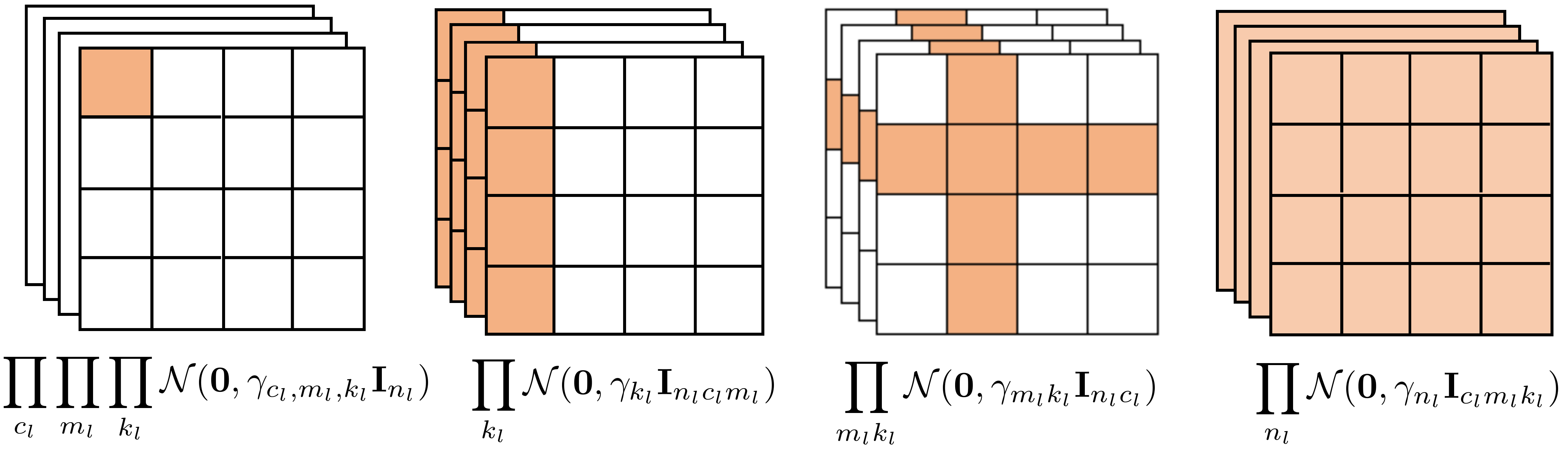}
    \vspace{-0.4cm}
    \caption{Structural Sparsity} 
    \vspace{-0.6cm}
    \label{fig:2}
\end{figure}

\vspace{3cm}
\subsection{Optimization Algorithm}
\label{app:edge_wised optimization algorithm}
As analyzed before, the optimization objective of searching architecture becomes removing redundant edges. The training algorithm is iteratively indexed by $t$. Each iteration may contain several epochs. The pseudo code is summarized in Algorithm~\ref{algorithm}. The cost function is simply maximum likelihood over the data $D$ with regularization whose intensity is controlled by the re-weighted coefficient $\omega$ 

\vspace{-0.5cm}
{\small
\begin{align}
    \mathcal{L}_{D} = {E_D}(\cdot) +  \lambda_w \sum\limits_{j<k} \sum\limits_{o' \in \mathcal{O}}  \|{\omega_{jk}^{o'}(t)  w_{jk}^{o'}}\|_1 + \lambda \|\boldsymbol{\mathcal{W}}\|_2^2
    \label{algo:eq1}
\end{align}
}The derivation can be found in Appendix \ref{general:prior:sec:costfunction} and \ref{algrithm procedures}. The algorithm mainly includes five parts.
The first part is to jointly train $\mathcal{W}$ and $\mathbf{w}$. The second part is to freeze the architecture parameters and prepare to compute their Hessian. The third part is to update the variables associated with the architecture parameters. The fourth part is to prune the architecture parameters and the pruned net will be trained in a standard way in the fifth part. As discussed previously on the drawback of magnitude based pruning metric, 
\begin{algorithm}[t]
\caption{
BayesNAS Algorithm. 
}
\begin{algorithmic}
    \REQUIRE 
    $\boldsymbol{\gamma}(0), \boldsymbol{\omega}(0), \mathbf{w}(0) = \mathbf{1}$; $\lambda=0.01$; sparsity intensity $\lambda_{w}^o \in \R^{+}$
    \ENSURE 
    \FOR{$t=1$ to $T_{\max}$}
            \STATE 1. Update $\mathbf{w}$ and $\mathcal{W}$ by minimizing $\mathcal{L}_D$ in \eqref{algo:eq1}
            \STATE 2. Compute Hessian for $\mathbf{w}$ (\eqref{eq:pre_Hessian_cat}, \ref{eq:scalar_pre_Hessian}, \ref{eq:scalar_pre_Hessian_approximation})  
            \STATE
            3. Update variables associated with $\mathbf{w}$
            \WHILE{$i<j<k, o, o'\in \mathcal{O}$}
                \STATE 
                \vspace{-0.6cm}
                {\small
                \begin{align}
                \label{eq:hyper_paras_update} 
                & C_{jk}^{o'}(t) = \left(\frac{1}{{\gamma_{jk}^{o'}}(t-1)} + {\mathbf{H}_{jk}^{o'}(t)}\right)^{-1}
                \\
                & {\omega}_{jk}^{o'}(t) =\frac{ \sqrt{{\gamma}_{jk}^{o'}(t-1)- {\rm C}_{jk}^{o'}(t)}}{{\gamma}_{jk}^{o'}(t-1)}
                \\
                & s_{jk}^{o'}(t) = \abs{{\frac{{w}_{jk}^{o'}(t)}{{\omega_{jk}^{o'}}(t)}}} 
                \\
                & {\gamma_{jk}^{o'}}(t) \text{ is given by } \ref{eq:gamma} \text{ or } \ref{eq:gamma-2}
                \end{align}}
        \ENDWHILE
    \STATE {4. Prune the architecture if the  entropy $\frac{\ln(2\pi e {\gamma_{jk}^{o'}})}{2} \leq 0 $}
    \STATE {5. Fix $\mathbf{w} = \mathbf{1}$, train the pruned net in the standard way}
    \ENDFOR
\end{algorithmic}
\label{algorithm}
\end{algorithm}  
we propose a new metric based on maximum entropy of the distribution. Since $p({w_{jk}^{o'}})$ in~\eqref{eq:ppath} is Gaussian with zero mean ${\gamma_{jk}^{o'}}$ variance, the maximum entropy is $\frac{1}{2}\ln(2\pi e{\gamma_{jk}^{o'}})$. We set the threshold for ${\gamma_{jk}^{o'}}$ to prune related edges when $\frac{1}{2}\ln(2\pi e{\gamma_{jk}^{o'}}) \leq 0$, \textit{i.e.}, ${\gamma_{jk}^{o'}} \leq 0.0585$.

The algorithm can be easily transferred to other scenarios. One scenario involves proxy tasks to find the cell. Similar to \eqref{algo:eq1}, we group same edge/operation in the repeated stacked cells where $g$ is the index. The cost function for proxy tasks is then given as follows in the form of re-weighted group Lasso:

\vspace{-0.6cm}
{\small
\begin{align}
\mathcal{L}_{D} = {E_D}(\cdot) +  \lambda_w \sum_{g}\sum\limits_{j<k} \sum\limits_{o' \in \mathcal{O}}  \|{\omega_{jk,g}^{o'}(t)  w_{jk,g}^{o'}}\|_2 + \lambda \|{\mathcal{W}}\|_2^2
\end{align}}

\vspace{-0.5cm}
The details are summarized in Algorithm \ref{algorithm:proxy_cell} of Appendix~\ref{app:algorithm for proxy tasks}. Another scenario is on Network Compression with structural sparsity, which is summarized in Algorithm \ref{algorithm:structural_sparsity} of Appendix \ref{app:structural compression algorithm}.

\begin{table*}
\caption{Classification errors of BayesNAS and state-of-the-art image classifiers on CIFAR-10.}
\label{table:evalcifar}
\begin{center}
\newcommand{\tabincell}[2]{\begin{tabular}{@{}#1@{}}#2\end{tabular}}
\resizebox{0.935\textwidth}{!}{
\begin{tabular}{lcccll}
\bottomrule
\multicolumn{1}{c}{\bf Architecture}  &\multicolumn{1}{c}{\tabincell{c}{\bf Test Error\\\bf (\%)}}  &\multicolumn{1}{c}{\tabincell{c}{\bf Params\\\bf (M)}} &\multicolumn{1}{c}{\tabincell{c}{\bf Search Cost\\\bf (GPU days)}} 
&\multicolumn{1}{c}{\tabincell{c}{\bf Search\\\bf Method}}
\\ \hline \vspace{-0.2cm}\\
DenseNet-BC \citep{huang2017densely}  &3.46    &25.6   &-   &manual  \\ 
\hline \vspace{-0.2cm} \\
NASNet-A + cutout \citep{zoph2017learning}  &2.65     &3.3   &1800    &RL \\
AmoebaNet-B + cutout \citep{real2018regularized}   &2.55 $\pm$ 0.05  &2.8 &3150 &evolution \\
Hierarchical Evo \citep{liu2017hierarchical}   &3.75 $\pm$ 0.12  &15.7 &300 &evolution \\
PNAS \citep{liu2018progressive}   &3.41 $\pm$ 0.09   &3.2     &225   &SMBO \\
ENAS + cutout \citep{pham2018efficient}   &2.89     &4.6   &0.5    &RL \\
\hline \vspace{-0.2cm} \\
Random search baseline + cutout \citep{liu2018darts}    &3.29 $\pm$ 0.15     &3.2     &1     &random     \\
DARTS (2nd order bi-level) + cutout \citep{liu2018darts}   & 2.76 $\pm$ 0.09   &3.4   &1  &gradient \\
SNAS (single-level) + moderate con + cutout \citep{xie2019snas}  &2.85 $\pm$ 0.02    &2.8  &1.5  &gradient  \\
DSO-NAS-share+cutout \citep{zhang2018single} & 2.84 $\pm$ 0.07 & 3.0 &1 &gradient \\
Proxyless-G + cutout  \cite{cai2019iclr} & 2.08 & 5.7 &- &gradient \\
\hline \vspace{-0.2cm}\\
BayesNAS  + cutout + $\lambda_w^o = 0.01$                            & 3.02$\pm$0.04 &  2.59$\pm$0.23  & 0.2 & gradient \\ 
BayesNAS  + cutout + $\lambda_w^o = 0.007$                            & 2.90$\pm$0.05 &  3.10$\pm$0.15  & 0.2 & gradient \\ 
BayesNAS  + cutout + $\lambda_w^o = 0.005$                            & 2.81$\pm$0.04 &  3.40$\pm$0.62  & 0.2 & gradient \\ 
BayesNAS  + TreeCell-A + Pyrimaid backbone + cutout                             & 2.41 &  3.4  & 0.1 & gradient \\\bottomrule
\vspace{-0.8cm}
\end{tabular}}
\end{center}
\end{table*}

\begin{table*}[ht]
	\centering
	\caption{Comparison with state-of-the-art image classifiers on ImageNet in the mobile setting.}
	\label{tab:imagenet-results}
 \resizebox{0.86\textwidth}{!}{
	\begin{tabular}{@{}lccccccc@{}}
		\toprule
		\multirow{2}{*}{\bf Architecture} & \multicolumn{2}{c}{\textbf{Test Error (\%)}}      & \textbf{Params}   & \textbf{Search Cost}  & \multirow{2}{*}{\bf Search Method} \\ \cline{2-3}
		& top-1 & top-5      & \textbf{(M)}   & \textbf{(GPU days)} & \textbf{} \\ \midrule
		Inception-v1 \citep{szegedy2015going} & 30.2 & 10.1 & 6.6  & -- & manual \\
		MobileNet \citep{howard2017mobilenets} & 29.4 & 10.5 & 4.2  & -- & manual \\ 
		ShuffleNet 2$\times$ (v1) \citep{zhang2018shufflenet} & 29.1 & 10.2 & $\sim$5  & -- & manual \\
		ShuffleNet 2$\times$ (v2) \citep{zhang2018shufflenet} & 26.3 & -- & $\sim$5  & -- & manual \\ \midrule
		NASNet-A \citep{zoph2017learning} & 26.0 & 8.4 & 5.3  & 1800 & RL \\
		NASNet-B \citep{zoph2017learning} & 27.2 & 8.7 & 5.3  & 1800 & RL \\
		NASNet-C \citep{zoph2017learning} & 27.5 & 9.0 & 4.9  & 1800 & RL \\
		AmoebaNet-A \citep{real2018regularized} & 25.5 & 8.0 & 5.1  & 3150 & evolution \\
		AmoebaNet-B \citep{real2018regularized} & 26.0 & 8.5 & 5.3  & 3150 & evolution \\
		AmoebaNet-C \citep{real2018regularized} & 24.3 & 7.6 & 6.4  & 3150 & evolution \\
		PNAS \citep{liu2018progressive} & 25.8 & 8.1 & 5.1  & $\sim$225 & SMBO \\ 
		DARTS \cite{liu2018darts}    &  26.9   &  9.0 & 4.9  & 4 & gradient \\ \midrule
		BayesNAS ($\lambda_w^o = 0.01$)    &  28.1   &  9.4 & 4.0  & 0.2 & gradient \\
		BayesNAS ($\lambda_w^o = 0.007$)    &  27.3   &  8.4 & 3.3  & 0.2 & gradient \\
		BayesNAS ($\lambda_w^o = 0.005$)    &  26.5   &  8.9 & 3.9  & 0.2 & gradient \\
		\bottomrule
	\end{tabular}
}
\end{table*}
\section{Experiments}
The experiments focus on two scenarios in NAS: proxy NAS and proxyless NAS. For proxy NAS, we follow the pipeline in DARTS \cite{liu2018darts} and SNAS \citep{xie2019snas}. First BayesNAS is applied to search for the best convolutional cells in a complete network on CIFAR-10. Then a network constructed by stacking learned cells is retrained for performance comparison. For proxyless NAS, we follow the pipeline in ProxylessNAS \cite{cai2019iclr}. First, the tree-like cell from \cite{cai2018path} with multiple paths is integrated into the PyramidNet \cite{han2017deep}. Then we search for the optimal path(s) within each cell by BayesNAS. Finally, the network is reconstructed by retaining only the optimal path(s) and retrained on CIFAR-10 for performance comparison. Detailed experiments setting is in Appendix \ref{datapr}. 
\subsection{Proxy Search}
\paragraph{Motivation}
Unlike DARTS and SNAS that rely on validation accuracy during or after search, we use $\gamma$ in BayesNAS as performance evaluation criterion which enables us to achieve it in an one-shot manner.
\vspace{1.5cm}
\begin{figure}[ht]
  \centering
    \includegraphics[scale = 0.18]{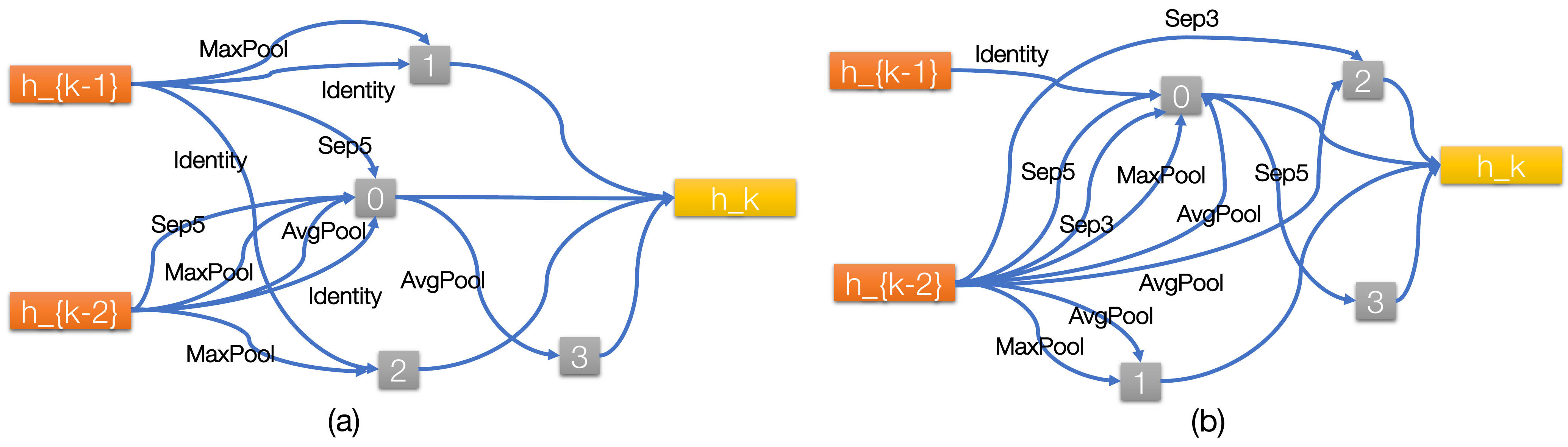}
  \caption{Normal and reduction cell found by BayesNAS with $\lambda_w^o = 0.01$.}
  \label{fig:3}
\end{figure}
\paragraph{Search Space}
Our setup follows DARTS and SNAS, where convolutional cells of 7 nodes are stacked for multiple times to form a network. The input nodes, \textit{i.e.}, the first and second nodes, of cell $k$ are set equal to the outputs of cell $k-1$ and cell $k-2$ respectively, with $1 \times 1$ convolutions inserted as necessary, and the output node is the depthwise concatenation of all the intermediate nodes. Reduction cells are located at the 1/3 and 2/3 of the total depth of the network to reduce the spatial resolution of feature maps.
Details about all operations included are shown in Appendix \ref{operations}. Unlike DARTS and SNAS, we exclude \emph{zero} operations.
\vspace{-0.2cm}
\paragraph{Training Settings}
In the searching stage, we train a small network stacked by 8 cells using BayesNAS with different $\lambda_w$. This network size is determined to fit into a single GPU. Since we cache the feature maps in memory, we can only set batch size as 18. The optimizer we use is SGD optimizer with momentum 0.9 and fixed learning rate 0.1. Other training setups follow DARTS and SNAS (Appendix \ref{details}). The search takes about $3$ hours on a single GPU\footnote{All the experiments were performed using NVIDIA TITAN V GPUs}.
\vspace{-0.2cm}
\paragraph{Search Results}
The normal and reduction cells learned on CIFAR-10 using BayesNAS are shown in Figure \ref{fig:3}a and \ref{fig:3}b. A large network of 20 cells where cells at 1/3 and 2/3 are reduction cells is trained from scratch with the batch size of 128. The validation accuracy is presented in Table \ref{table:evalcifar}. The test error rate of BayesNAS is competitive against state-of-the-art techniques and BayesNAS is able to find convolutional cells with fewer parameters when compared to DARTS and SNAS.
\begin{figure}
  \centering
      \includegraphics[scale=0.2]{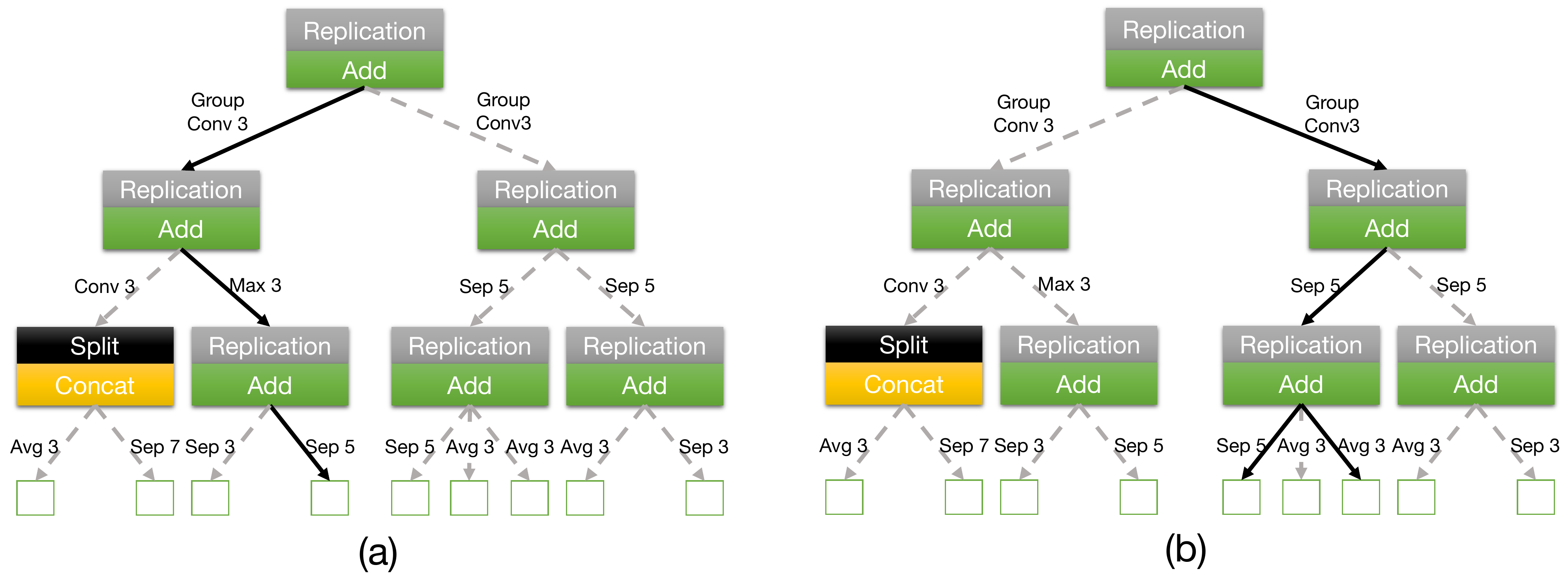}
  \caption{The pruned tree-cell: (a) The chain-like where only one path exists in the cell connecting the input of the cell to its output. (b) The inception structure where divergence and convergence both exist in the cell.
  The solid directed lines denote the path found by BayesNAS while the dashed ones denote the paths discarded.}
  \label{fig:treecell}
\end{figure}
\vspace{3cm}
\subsection{Proxyless Search}
\paragraph{Motivation}
Using existing tree-like cell, we apply BayesNAS to search for the optimal path(s) within each cell. Varying from proxy search, cells do not share architecture  in proxyless search. 
\vspace{-0.2cm}
\paragraph{Search Space}
The backbone used is PyramidNet with three layers each consisting of $18$ bottleneck blocks and $\alpha=84$. All $3 \times 3$ convolution in bottleneck blocks are replaced by the tree-cell that has in total $9$ possible paths within. The groups for grouped convolution is set to $2$. For the detailed structure of the tree-cell, we refer to \cite{cai2018path}.
\vspace{-0.2cm}
\paragraph{Training Settings}
In the searching stage, we set batch size to 32 and learning rate to 0.1. We use the same optimizer as for proxy search. The $\lambda$ of BayesNAS for each possible path is set to $1\times10^{-2}$.           
\vspace{-0.2cm}
\paragraph{Search Results}
Because each cell can have a different structure in proxyless setting, we demonstrate only two typical types of cell structure among all of them in Figure \ref{fig:treecell}a and Figure \ref{fig:treecell}b. The first type is a chain-like structure where only one path exists in the cell connecting the input of the cell to its output. The second type is an inception structure where divergence and convergence both exist in the cell. Our further observation reveals that some cells are dispensable with respect to the entire network. After the architecture is determined, the network is trained from scratch with the batch size of 64, learning rate as 0.1 and cosine annealing learning rate decay schedule \cite{loshchilov2016sgdr}. The validation accuracy is also presented in Table \ref{table:evalcifar}. Although test error increases slightly compared to \cite{cai2019iclr}, there is a significant drop in the number of model parameters to be learned which is beneficial for both training and inference.
\section{Transferability to ImageNet}
For ImageNet mobile setting, the input images are of size 224$\times$224. A network of 14 cells is trained for 250 epochs with batch size 128, weight decay $3 \times 10^{-5}$ and initial SGD learning rate 0.1 (decayed by a factor of 0.97 after each epoch). Results in Table~\ref{tab:imagenet-results} show that the cell learned on CIFAR-10 can be transfered to ImageNet and is capable of achieving competitive performance.

\section{Conclusion and Future Work}
We introduce BayesNAS that can directly learn a sparse neural network architecture. We significantly reduce the search time by using only one epoch to get the candidate architecture. Our current implementation is inefficient by caching all the feature maps in memory to compute the Hessian. However, Hessian computation can be done along with backpropagation which will potentially further reduce the searching time and scale our approach to larger search spaces.    
\newpage
\section*{Acknowledgements}
The work of Hongpeng Zhou is sponsored by the program of China Scholarships Council (No.201706120017).

\renewcommand\bibname{References}
\bibliography{ref}
\bibliographystyle{icml2019}

\newpage
\onecolumn
{\large
\begin{center}
{\centering \textbf{Appendix}}
\end{center}
}
\numberwithin{equation}{section}
\makeatletter
\renewcommand{\theequation}{S\arabic{equation}} 
\renewcommand{\thesection}{S\arabic{section}}   
\renewcommand{\thetable}{S\arabic{table}}   
\renewcommand{\thefigure}{S\arabic{figure}}
\numberwithin{equation}{subsection}

\appendix
\section{BayesNAS Algorithm Derivation}
\label{app:optimization algorithm derivation}
\subsection{Algorithm Derivation}
\label{general:prior:sec:costfunction}
In this subsection, we explain the detailed algorithm of updating hyper-parameters for the abstracted single motif as shown in Figure~\ref{fig:1}e. The proposition about optimization objective will be illustrated firstly.
\begin{proposition}
\label{general:proposition:cost}
Suppose the likelihood of the architecture parameters of a neural network $\mathbf{w}$ could be formulated as one exponential family distribution \begin{math} p(\mathbf{Y}\given \mathbf{w}, \mathbf{X},{\mathbf{s}}) \sim \exp \left(-E_D(\mathbf{Y}; \text{Net}(\mathbf{X}; \mathbf{w}); \mathbf{s})\right) \end{math}, where $\mathcal{D}=(\mathbf{X},\mathbf{Y})$ is the given dataset, $\mathbf{s}$ stands for the uncertainty and $E_D(*)$ represents the \emph{energy function} over data. 
The sparse prior with super Gaussian distribution for each architecture parameter has been defined in \eqref{eq:prior_weights_2}.
The unknown architecture parameter of the network $\mathbf{w}$
and hyperparameter ${\mathbf{s}}$ can be approximately obtained by solving the following optimization problem
\begin{equation}
    \min_{ \mathbf{w}, {\mathbf{s}}} \mathcal{L}( \mathbf{w}, {\mathbf{s}}) 
\end{equation}
specially, for the architecture parameter $\mathbf{w_{jk}^{o'}}$ which is associated with one operation of the edge $e_{jk}$ ($j<k$), the optimization problem could be reformulated as:   
\begin{equation}
\begin{aligned}
&\mathcal{L}( \mathbf{w}_{jk}^{o'}, {\mathbf{s}}_{jk}^{o'})
= {\mathbf{w}_{jk}^{o'}} \Hessian({\mathbf{w}_{jk}^{o'}}^* ) {\mathbf{w}_{jk}^{o'}}
+ 2 {\mathbf{w}_{jk}^{o'}} \left[\grad({\mathbf{w}_{jk}^{o'}}^* )-\Hessian({\mathbf{w}_{jk}^{o'}}^*) {\mathbf{w}_{jk}^{o'}}^{*}\right]
+{\mathbf{w}_{jk}^{o'}}  {{\mathbf{s}}_{jk}^{o'}}^{-1} {\mathbf{w}_{jk}^{o'}} 
\\
&+\log|{\mathbf{s}}_{jk}^{o'}|+\log|\Hessian({\mathbf{w}_{jk}^{o'}}^* )+ {{\mathbf{s}}_{jk}^{o'}}^{-1} |
-2\log{b({\mathbf{w}_{jk}^{o'}}^*)} 
\label{general:prior:eq:cost}
\end{aligned}
\end{equation}
where ${\mathbf{w}_{jk}^{o'}}^*$ is arbitrary, and 
$$\grad({\mathbf{w}_{jk}^{o'}}^* ) \define \nabla E_D({\mathbf{w}_{jk}^{o'}} )|_{{\mathbf{w}_{jk}^{o'}}^{*}},\Hessian({\mathbf{w}_{jk}^{o'}}^* ) \define  \nabla \nabla E_D({\mathbf{w}_{jk}^{o'}} ) |_{{\mathbf{w}_{jk}^{o'}}^{*}}$$
and

$$
b({\mathbf{w}_{jk}^{o'}}^* ) \define
\exp \left \{- \left(\frac{1}{2}{{\mathbf{w}_{jk}^{o'}}^*} \Hessian({\mathbf{w}_{jk}^{o'}}^* ) {{\mathbf{w}_{jk}^{o'}}^*} - {{\mathbf{w}_{jk}^{o'}}^*} \grad({\mathbf{w}_{jk}^{o'}}^*)
+E_D({\mathbf{w}_{jk}^{o'}}^* ) \right) \right \}
$$
\end{proposition}
It should also be noted that ${\mathbf{s}}_{jk}^{o'}$ represents the uncertainty of ${\mathbf{w}_{jk}^{o'}}$ without considering the dependency between edge $e_{jk}^{o'}$ and $\sum_{i<j}e_{ij}^{o}$, where $o'$ and $o$ stands for one possible operation in corresponding edges.

\begin{proof}
Given the likelihood with exponential family distribution $$p(\mathbf{Y}\given\mathbf{w}_{jk}^{o'}, \mathbf{X},{\mathbf{s}}_{jk}^{o'}) \sim \exp \left(-E_D(\mathbf{Y}; \text{Net}(\mathbf{X}; \mathbf{w}_{jk}^{o'}); {\mathbf{s}}_{jk}^{o'})\right)$$ 
as explained in \eqref{eq:ppath}, we define the prior of $w_{jk}^{o'}$ with Gaussian distribution 
$$p(w_{jk}^{o'}) = \bN(w_{jk}^{o'}|0, {\mathbf{s}}_{jk}^{o'})$$
The marginal likelihood could be calculated as: 
\begin{equation}
\begin{aligned}
p(\mathbf{Y}\given\mathbf{w}_{jk}^{o'})\mathcal{N}( {\mathbf{w}_{jk}^{o'}}|\mathbf{0},{{\mathbf{s}}_{jk}^{o'}})d{\mathbf{w}_{jk}^{o'}}
=  \int \exp\{-E_D({\mathbf{w}_{jk}^{o'}} )\}\bN( {\mathbf{w}_{jk}^{o'}}|\mathbf{0},{{\mathbf{s}}_{jk}^{o'}})d{\mathbf{w}_{jk}^{o'}} \\
\label{general:generalmarignal}
\end{aligned}
\end{equation}
Typically, this integral is intractable or has no analytical solution.

The mean and covariance can be fixed if the family is Gaussian. Performing a Taylor series expansion around some point ${\mathbf{w}_{jk}^{o'}}^{*}$, $E_D({\mathbf{w}_{jk}^{o'}})$ can be approximated as
\begin{equation}
\begin{aligned}
E_D({\mathbf{w}_{jk}^{o'}})
\approx E_D({\mathbf{w}_{jk}^{o'}}^*)+({\mathbf{w}_{jk}^{o'}}-{\mathbf{w}_{jk}^{o'}}^*)\grad({\mathbf{w}_{jk}^{o'}}^*)+\frac{1}{2}({\mathbf{w}_{jk}^{o'}}-{\mathbf{w}_{jk}^{o'}}^*) \Hessian({\mathbf{w}_{jk}^{o'}}^*) ({\mathbf{w}_{jk}^{o'}}-{\mathbf{w}_{jk}^{o'}}^*) 
\label{general:prior:Likelihood}
\end{aligned}
\end{equation}
where $\grad(\cdot)$ is the gradient and $\Hessian(\cdot)$ is the Hessian of the energy function $E$
\begin{subequations}
\begin{align}
\grad({\mathbf{w}_{jk}^{o'}}^*) &\define \nabla E_D({\mathbf{w}_{jk}^{o'}})|_{{\mathbf{w}_{jk}^{o'}}^{*}} 
\label{general:grad}\\
\Hessian({\mathbf{w}_{jk}^{o'}}^*) &\define  \nabla \nabla E_D({\mathbf{w}_{jk}^{o'}}) |_{{\mathbf{w}_{jk}^{o'}}^{*}}
\label{general:hessian}
\end{align}
\label{generarl:moment} 
\end{subequations}
To derive the cost function in~\eqref{general:prior:eq:cost}, we introduce the posterior mean and covariance:
\begin{subequations}
\begin{align}
\mean&= \variance \cdot \left[\grad({\mathbf{w}_{jk}^{o'}}^*) + \Hessian({\mathbf{w}_{jk}^{o'}}^*) {\mathbf{w}_{jk}^{o'}}^{*} \right], \label{general:prior:mean} \\
\variance&= \left[{{\mathbf{s}}_{jk}^{o'}}^{-1} + \Hessian({\mathbf{w}_{jk}^{o'}}^*) \right]^{-1}. \label{general:prior:variance} 
\end{align}
\label{general:prior:moment} 
\end{subequations}
Then define the following quantities
\begin{subequations}
\begin{align}
b({\mathbf{w}_{jk}^{o'}}^*) & \define \exp \left \{- \left(\frac{1}{2}{{\mathbf{w}_{jk}^{o'}}^*} \Hessian({\mathbf{w}_{jk}^{o'}}^*) {{\mathbf{w}_{jk}^{o'}}^*} - {{\mathbf{w}_{jk}^{o'}}^*} \grad({\mathbf{w}_{jk}^{o'}}^*)  + E_D({\mathbf{w}_{jk}^{o'}}^*) \right) \right \}, \label{general:abbv:b} \\
c({\mathbf{w}_{jk}^{o'}}^*) & \define \exp \left \{ \frac{1}{2} \hat{\grad}({\mathbf{w}_{jk}^{o'}}^*) \Hessian({\mathbf{w}_{jk}^{o'}}^*) \hat{\grad}({\mathbf{w}_{jk}^{o'}}^*) 
\right \},  \label{general:abbv:c} \\
d({\mathbf{w}_{jk}^{o'}}^*) & \define \sqrt{|\Hessian({\mathbf{w}_{jk}^{o'}}^*)|}, \label{general:abbv:d} \\
\hat{\grad}({\mathbf{w}_{jk}^{o'}}^*)& \define \grad({\mathbf{w}_{jk}^{o'}}^*)-\Hessian({\mathbf{w}_{jk}^{o'}}^*) {\mathbf{w}_{jk}^{o'}}^{*}.
\label{general:abbv:g} 
\end{align}
\label{general:abbv:bcdg} 
\end{subequations}
Now the approximated likelihood $p(\mathbf{Y}|{\mathbf{w}_{jk}^{o'}})$ is a exponential of quadratic, then Gaussian,
\begin{equation}
\begin{aligned}
& p(\mathbf{Y}|{\mathbf{w}_{jk}^{o'}}) \\
=& \exp\{-E_D({\mathbf{w}_{jk}^{o'}})\} \\
\approx & \exp \left \{-\left(\quadratic+E_D({\mathbf{w}_{jk}^{o'}}^*)\right)\right \} \\
= & \exp \left \{-\left(\frac{1}{2}{\mathbf{w}_{jk}^{o'}} \Hessian({\mathbf{w}_{jk}^{o'}}^*) {\mathbf{w}_{jk}^{o'}}+ {\mathbf{w}_{jk}^{o'}} \left[\grad({\mathbf{w}_{jk}^{o'}}^*)-\Hessian({\mathbf{w}_{jk}^{o'}}^*) {\mathbf{w}_{jk}^{o'}}^{*}\right]\right) \right \} \\
& \cdot \exp \left \{- \left(\frac{1}{2}{{\mathbf{w}_{jk}^{o'}}^*} \Hessian({\mathbf{w}_{jk}^{o'}}^*) {{\mathbf{w}_{jk}^{o'}}^*} - {{\mathbf{w}_{jk}^{o'}}^*} \grad({\mathbf{w}_{jk}^{o'}}^*)  + E_D({\mathbf{w}_{jk}^{o'}}^*) \right) \right \} \\
= & b({\mathbf{w}_{jk}^{o'}}^*) \cdot \exp \left \{-\left(\frac{1}{2}{\mathbf{w}_{jk}^{o'}} \Hessian({\mathbf{w}_{jk}^{o'}}^*) {\mathbf{w}_{jk}^{o'}}+ {\mathbf{w}_{jk}^{o'}} \hat{\grad}({\mathbf{w}_{jk}^{o'}}^*)\right) \right \} \\
& \cdot \exp \left \{\frac{1}{2}\hat{\grad}({\mathbf{w}_{jk}^{o'}}^*) \Hessian({\mathbf{w}_{jk}^{o'}}^*) \hat{\grad}({\mathbf{w}_{jk}^{o'}}^*) - \frac{1}{2} \hat{\grad}({\mathbf{w}_{jk}^{o'}}^*) \Hessian({\mathbf{w}_{jk}^{o'}}^*) \hat{\grad}({\mathbf{w}_{jk}^{o'}}^*)\right \} \\
= & b({\mathbf{w}_{jk}^{o'}}^*) \cdot c({\mathbf{w}_{jk}^{o'}}^*) \\
& \cdot \exp \left \{-\left(\frac{1}{2}{\mathbf{w}_{jk}^{o'}} \Hessian({\mathbf{w}_{jk}^{o'}}^*) {\mathbf{w}_{jk}^{o'}}+ {\mathbf{w}_{jk}^{o'}} \hat{\grad}({\mathbf{w}_{jk}^{o'}}^*) + \frac{1}{2}\hat{\grad}({\mathbf{w}_{jk}^{o'}}^*) \Hessian({\mathbf{w}_{jk}^{o'}}^*) \hat{\grad}({\mathbf{w}_{jk}^{o'}}^*)\right) \right \} \\
= & (2 \pi)^{\m/2}  b({\mathbf{w}_{jk}^{o'}}^*) c({\mathbf{w}_{jk}^{o'}}^*) d({\mathbf{w}_{jk}^{o'}}^*) \cdot \mathcal{N}({\mathbf{w}_{jk}^{o'}}|\hat{{\mathbf{w}_{jk}^{o'}}}^{*}, \Hessian^{-1}({\mathbf{w}_{jk}^{o'}}^*)) \\
\define & A({\mathbf{w}_{jk}^{o'}}^*) \cdot \mathcal{N}({\mathbf{w}_{jk}^{o'}}|\hat{{\mathbf{w}_{jk}^{o'}}}^{*}, \Hessian^{-1}({\mathbf{w}_{jk}^{o'}}^*)),
\label{likelihood-approximation}
\end{aligned}
\end{equation}
where
\begin{equation}
\begin{aligned}
A({\mathbf{w}_{jk}^{o'}}^*) &= (2 \pi)^{\m/2}  b({\mathbf{w}_{jk}^{o'}}^*) c({\mathbf{w}_{jk}^{o'}}^*) d({\mathbf{w}_{jk}^{o'}}^*), \\ 
\hat{{\mathbf{w}_{jk}^{o'}}}^{*} & = -\Hessian^{-1}({\mathbf{w}_{jk}^{o'}}^*) \hat{\grad}({\mathbf{w}_{jk}^{o'}}^*) = {\mathbf{w}_{jk}^{o'}}^{*} -\Hessian^{-1}({\mathbf{w}_{jk}^{o'}}^*)\grad({\mathbf{w}_{jk}^{o'}}^*).
\notag
\end{aligned}
\end{equation}
We can write the approximate marginal likelihood as
\begin{equation}
\begin{aligned}
& A({\mathbf{w}_{jk}^{o'}}^*) \int \mathcal{N}({\mathbf{w}_{jk}^{o'}}|\hat{{\mathbf{w}_{jk}^{o'}}}^{*}, \Hessian^{-1}({\mathbf{w}_{jk}^{o'}}^*)) \cdot \mathcal{N}( {\mathbf{w}_{jk}^{o'}}|\mathbf{0},{{\mathbf{s}}_{jk}^{o'}})d{\mathbf{w}_{jk}^{o'}} \\
= & b({\mathbf{w}_{jk}^{o'}}^*) \cdot \int \exp \left \{-\left(\frac{1}{2}{\mathbf{w}_{jk}^{o'}} \Hessian({\mathbf{w}_{jk}^{o'}}^*) {\mathbf{w}_{jk}^{o'}}+ {\mathbf{w}_{jk}^{o'}} \hat{\grad}({\mathbf{w}_{jk}^{o'}}^*)\right) \right \}\mathcal{N}( {\mathbf{w}_{jk}^{o'}}|\mathbf{0},{{\mathbf{s}}_{jk}^{o'}})d{\mathbf{w}_{jk}^{o'}} \\
=&\frac{b({\mathbf{w}_{jk}^{o'}}^*)}{\left(2\pi\right)^{1/2}|{{\mathbf{s}}_{jk}^{o'}}|^{1/2}}
\int \exp\{-\hat{E}({\mathbf{w}_{jk}^{o'}})\}d{\mathbf{w}_{jk}^{o'}}
,
\label{general:prior:integral}
\end{aligned}
\end{equation}
where
\begin{equation}
\begin{aligned}
\hat{E}({\mathbf{w}_{jk}^{o'}})=\frac{1}{2}{\mathbf{w}_{jk}^{o'}} \Hessian({\mathbf{w}_{jk}^{o'}}^*) {\mathbf{w}_{jk}^{o'}}+ {\mathbf{w}_{jk}^{o'}} \hat{\grad}({\mathbf{w}_{jk}^{o'}}^*)+\frac{1}{2} {\mathbf{w}_{jk}^{o'}} {{\mathbf{s}}_{jk}^{o'}}^{-1} {\mathbf{w}_{jk}^{o'}}.
\end{aligned}
\end{equation}
Equivalently, we get
\begin{equation}
\begin{aligned}
\hat{E}({\mathbf{w}_{jk}^{o'}})=\frac{1}{2}({\mathbf{w}_{jk}^{o'}}-\mean)(\variance)^{-1}({\mathbf{w}_{jk}^{o'}}-\mean)+\hat{E}(\mathbf{Y}),
\label{general:prior:integral2}
\end{aligned}
\end{equation}
where $\mean$ and $\variance$ are given in \eqref{general:prior:moment}. From~\eqref{general:prior:mean} and~\eqref{general:prior:variance}, the data-dependent term can be re-expressed as
\begin{equation}
\begin{aligned}
\hat{E}(\mathbf{Y})
=&  {\frac{1}{2}\mean \Hessian({\mathbf{w}_{jk}^{o'}}^*) \mean+  \mean\grad({\mathbf{w}_{jk}^{o'}}^*)}+ \frac{1}{2}\mean{{\mathbf{s}}_{jk}^{o'}}^{-1}\mean \\
=&\min_{{\mathbf{w}_{jk}^{o'}}}  \left[\frac{1}{2}{\mathbf{w}_{jk}^{o'}} \Hessian({\mathbf{w}_{jk}^{o'}}^*) {\mathbf{w}_{jk}^{o'}}+ {\mathbf{w}_{jk}^{o'}} \hat{\grad}({\mathbf{w}_{jk}^{o'}}^*) + \frac{1}{2}{\mathbf{w}_{jk}^{o'}}  {{\mathbf{s}}_{jk}^{o'}}^{-1} {\mathbf{w}_{jk}^{o'}}\right] \\
=&\min_{{\mathbf{w}_{jk}^{o'}}}  \left[\frac{1}{2}{\mathbf{w}_{jk}^{o'}} \Hessian({\mathbf{w}_{jk}^{o'}}^*) {\mathbf{w}_{jk}^{o'}}+ {\mathbf{w}_{jk}^{o'}} \left(\grad({\mathbf{w}_{jk}^{o'}}^*)-\Hessian({\mathbf{w}_{jk}^{o'}}^*) {\mathbf{w}_{jk}^{o'}}^{*}\right) + \frac{1}{2}{\mathbf{w}_{jk}^{o'}}  {{\mathbf{s}}_{jk}^{o'}}^{-1} {\mathbf{w}_{jk}^{o'}}\right].
\label{general:prior:data-dependent-term}
\end{aligned}
\end{equation}
Using~\eqref{general:prior:integral2}, we can evaluate the integral in~\eqref{general:prior:integral} to obtain
\begin{equation}
\begin{aligned}
\int \exp \left \{-\hat{E}({\mathbf{w}_{jk}^{o'}}) \right \}d{\mathbf{w}_{jk}^{o'}}=\exp \left \{-\hat{E}(\mathbf{Y}) \right \} 2\pi|\variance|^{1/2}.
\end{aligned}
\end{equation}
Applying a $-2\log(\cdot)$ transformation to~\eqref{general:prior:integral}, we have
\begin{equation}
\begin{aligned}
&-2\log\left[\frac{b({\mathbf{w}_{jk}^{o'}}^*)}{\left(2\pi\right)^{1/2}|{{\mathbf{s}}_{jk}^{o'}}|^{1/2}}
\int \exp\{-\hat{E}({\mathbf{w}_{jk}^{o'}})\}d{\mathbf{w}_{jk}^{o'}} \right] \\
\propto & -2\log {b({\mathbf{w}_{jk}^{o'}}^*)} + \hat{E}(\mathbf{Y})+
\\
&\log  |{{\mathbf{s}}_{jk}^{o'}}|+\log |\Hessian({\mathbf{w}_{jk}^{o'}}^*)+ {{\mathbf{s}}_{jk}^{o'}}^{-1} |
\\
\propto & {\mathbf{w}_{jk}^{o'}} \Hessian({\mathbf{w}_{jk}^{o'}}^*) {\mathbf{w}_{jk}^{o'}}+ 2{\mathbf{w}_{jk}^{o'}} \hat{\grad}({\mathbf{w}_{jk}^{o'}}^*)+ {\mathbf{w}_{jk}^{o'}}  {{\mathbf{s}}_{jk}^{o'}}^{-1} {\mathbf{w}_{jk}^{o'}}\\
& +\log  |{{\mathbf{s}}_{jk}^{o'}}|+\log |\Hessian({\mathbf{w}_{jk}^{o'}}^*)+ {{\mathbf{s}}_{jk}^{o'}}^{-1} |-2\log {b({\mathbf{w}_{jk}^{o'}}^*)}.
\label{general:prior:integral3}
\end{aligned}
\end{equation}
Therefore we get the following cost function to be minimised in~\eqref{general:prior:eq:cost} over ${\mathbf{w}_{jk}^{o'}}, {{\mathbf{s}}_{jk}^{o'}}, $
\begin{equation}
\begin{aligned}
&\mathcal{L}({\mathbf{w}_{jk}^{o'}}, {{\mathbf{s}}_{jk}^{o'}}) 
=  {\mathbf{w}_{jk}^{o'}} \Hessian({\mathbf{w}_{jk}^{o'}}^*) {\mathbf{w}_{jk}^{o'}}+ 2 {\mathbf{w}_{jk}^{o'}} \left[\grad({\mathbf{w}_{jk}^{o'}}^*)-\Hessian({\mathbf{w}_{jk}^{o'}}^*) {\mathbf{w}_{jk}^{o'}}^{*}\right] +{\mathbf{w}_{jk}^{o'}}  {{\mathbf{s}}_{jk}^{o'}}^{-1} {\mathbf{w}_{jk}^{o'}} 
\\
&+\log  |{{\mathbf{s}}_{jk}^{o'}}|+\log |\Hessian({\mathbf{w}_{jk}^{o'}}^*)+ {{\mathbf{s}}_{jk}^{o'}}^{-1} |
-2\log {b({\mathbf{w}_{jk}^{o'}}^*)}.
\notag
\end{aligned}
\end{equation}
It can be easily found that the first line of $\mathcal{L}$ is quadratic programming with $\ell_2$ regularizer. The second line is all about the hyperparameter ${{\mathbf{s}}_{jk}^{o'}}$.

Once the estimation on ${\mathbf{w}_{jk}^{o'}}$ and ${{\mathbf{s}}_{jk}^{o'}}$ are obtained, the cost function is alternatively optimised. The new estimated ${\mathbf{w}_{jk}^{o'}}$ can substitute ${\mathbf{w}_{jk}^{o'}}^{*}$ and repeat the estimation iteratively.

\hfill $\blacksquare$
\end{proof}

We note that in~\eqref{generarl:moment}, ${\mathbf{w}_{jk}^{o'}}^{*}$ may not be the mode (i.e., the lowest energy state), which means the gradient term $\grad$ may not be zero. Therefore the selection of ${{\mathbf{w}_{jk}^{o'}}(1)}^{*}$ remains to be problematic. We give the following Corollary to address this issue, which is more general.
\begin{corollary}
Suppose
\begin{equation}
\begin{aligned}
{\mathbf{w}_{jk}^{o'}}^{*} = \argmin_{{\mathbf{w}_{jk}^{o'}}} E_D({\mathbf{w}_{jk}^{o'}})+{\mathbf{w}_{jk}^{o'}}  {{\mathbf{s}}_{jk}^{o'}}^{-1} {\mathbf{w}_{jk}^{o'}}, 
\label{}
\end{aligned}
\end{equation}
we define a new cost function
\begin{equation}
\begin{aligned}
& \hat{\mathcal{L}}({\mathbf{w}_{jk}^{o'}}, {{\mathbf{s}}_{jk}^{o'}})
\define E_D({\mathbf{w}_{jk}^{o'}}) +{\mathbf{w}_{jk}^{o'}}  {{\mathbf{s}}_{jk}^{o'}}^{-1} {\mathbf{w}_{jk}^{o'}} 
 +\log  |{{\mathbf{s}}_{jk}^{o'}}|+\log |\Hessian({\mathbf{w}_{jk}^{o'}}^*)+ {{\mathbf{s}}_{jk}^{o'}}^{-1} |
-2\log {b({\mathbf{w}_{jk}^{o'}}^*)}.
\label{general:generalcost}
\end{aligned}
\end{equation}
Instead of minimising $\mathcal{L}({\mathbf{w}_{jk}^{o'}}, {{\mathbf{s}}_{jk}^{o'}})$, we can solve the following optimization problem to get ${\mathbf{w}_{jk}^{o'}}, {{\mathbf{s}}_{jk}^{o'}}, $
$$\min_{{\mathbf{w}_{jk}^{o'}}, {{\mathbf{s}}_{jk}^{o'}}, } \hat{\mathcal{L}}({\mathbf{w}_{jk}^{o'}}, {{\mathbf{s}}_{jk}^{o'}}).$$
\end{corollary}

\begin{proof}
Since the likelihood is
$$p(\mathbf{Y}|{\mathbf{w}_{jk}^{o'}}) = \exp\{-E_D({\mathbf{w}_{jk}^{o'}})\},$$
then
$$\min_{{\mathbf{w}_{jk}^{o'}}} E_D({\mathbf{w}_{jk}^{o'}})+{\mathbf{w}_{jk}^{o'}}  {{\mathbf{s}}_{jk}^{o'}}^{-1} {\mathbf{w}_{jk}^{o'}} $$ is exactly the regularised \emph{maximum likelihood estimation} with $\ell_2$ type regulariser.

We look at the first part of $\mathcal{L}({\mathbf{w}_{jk}^{o'}}, {{\mathbf{s}}_{jk}^{o'}})$ in~\eqref{general:prior:eq:cost}, and define them as  
$$\mathcal{L}_0({\mathbf{w}_{jk}^{o'}}, {{\mathbf{s}}_{jk}^{o'}}) \define {\mathbf{w}_{jk}^{o'}} \Hessian({\mathbf{w}_{jk}^{o'}}^*) {\mathbf{w}_{jk}^{o'}}+ 2 {\mathbf{w}_{jk}^{o'}} \left[\grad({\mathbf{w}_{jk}^{o'}}^*)-\Hessian({\mathbf{w}_{jk}^{o'}}^*) {\mathbf{w}_{jk}^{o'}}^{*}\right]+{\mathbf{w}_{jk}^{o'}}  {{\mathbf{s}}_{jk}^{o'}}^{-1} {\mathbf{w}_{jk}^{o'}} ,$$
then
\begin{equation}
\begin{aligned}
&\min_{{\mathbf{w}_{jk}^{o'}}} \mathcal{L}_0({\mathbf{w}_{jk}^{o'}}, {{\mathbf{s}}_{jk}^{o'}}) \\
=& \min_{{\mathbf{w}_{jk}^{o'}}} \quadratic+E_D({\mathbf{w}_{jk}^{o'}}^*)+{\mathbf{w}_{jk}^{o'}}  {{\mathbf{s}}_{jk}^{o'}}^{-1} {\mathbf{w}_{jk}^{o'}} \\
\approx &  \min_{{\mathbf{w}_{jk}^{o'}}} E_D({\mathbf{w}_{jk}^{o'}})+{\mathbf{w}_{jk}^{o'}}  {{\mathbf{s}}_{jk}^{o'}}^{-1} {\mathbf{w}_{jk}^{o'}} 
\label{}
\end{aligned}
\end{equation}
where, given~\eqref{generarl:moment},
\begin{equation}
\begin{aligned}
\grad({\mathbf{w}_{jk}^{o'}}) &= \nabla E_D({\mathbf{w}_{jk}^{o'}})  \\
\Hessian({\mathbf{w}_{jk}^{o'}})&=  \nabla \nabla E_D({\mathbf{w}_{jk}^{o'}}). 
\label{}
\end{aligned}
\end{equation}
Such quadratic approximation to $E_D({\mathbf{w}_{jk}^{o'}})+{\mathbf{w}_{jk}^{o'}}  {{\mathbf{s}}_{jk}^{o'}}^{-1} {\mathbf{w}_{jk}^{o'}} $ is actually the same as the approximation procedure in Trust-Region Methods where a region is defined around the current iterate within which they trust the model to be an adequate representation of the objective function \citep[pp.65]{NoceWrig06}.

To obtain each step, we seek a solution of the subproblem at iteration $t$
\begin{equation}
\begin{aligned}
&  \min_{{\mathbf{w}_{jk}^{o'}}} E_D({\mathbf{w}_{jk}^{o'}}(t-1))+{\mathbf{w}_{jk}^{o'}}  {{\mathbf{s}}_{jk}^{o'}(t-1)}^{-1} {\mathbf{w}_{jk}^{o'}}  \\
=& \min_{{\mathbf{w}_{jk}^{o'}}} {\frac{1}{2}({\mathbf{w}_{jk}^{o'}}-{\mathbf{w}_{jk}^{o'}}(t-1)) \Hessian({\mathbf{w}_{jk}^{o'}}(t-1)) ({\mathbf{w}_{jk}^{o'}}-{\mathbf{w}_{jk}^{o'}}(t-1)) +  ({\mathbf{w}_{jk}^{o'}}-{\mathbf{w}_{jk}^{o'}}(t-1))\grad({\mathbf{w}_{jk}^{o'}}(t-1))} \\
& +{\mathbf{w}_{jk}^{o'}}  {{\mathbf{s}}_{jk}^{o'}(t-1)}^{-1} {\mathbf{w}_{jk}^{o'}}
\end{aligned}
\end{equation}
Suppose
$${\mathbf{w}_{jk}^{o'}}^{*} = \argmin_{{\mathbf{w}_{jk}^{o'}}} E_D({\mathbf{w}_{jk}^{o'}})+{\mathbf{w}_{jk}^{o'}}  {{\mathbf{s}}_{jk}^{o'}}^{-1} {\mathbf{w}_{jk}^{o'}}, $$
then inject ${\mathbf{w}_{jk}^{o'}}^{*}$ into $\min_{{\mathbf{w}_{jk}^{o'}}, {{\mathbf{s}}_{jk}^{o'}}, } \mathcal{L}({\mathbf{w}_{jk}^{o'}}, {{\mathbf{s}}_{jk}^{o'}})$,
we can optimise~\eqref{general:generalcost} 
instead of~\eqref{general:prior:eq:cost}, i.e., 
$\min_{{\mathbf{w}_{jk}^{o'}}, {{\mathbf{s}}_{jk}^{o'}}, } \hat{\mathcal{L}}({\mathbf{w}_{jk}^{o'}}, {{\mathbf{s}}_{jk}^{o'}})$.

\hfill $\blacksquare$

\end{proof}
\subsection{Algorithm for Proxyless Tasks}
\label{algrithm procedures}
In this Section, we propose iterative optimization algorithms to estimate ${\mathbf{w}_{jk}^{o'}}$ and ${{\mathbf{s}}_{jk}^{o'}}$ alternatively.

\subsubsection{Optimization for architecture parameter ${\mathbf{w}_{jk}^{o'}}$ and switch ${{\mathbf{s}}_{jk}^{o'}}$}
\label{algorithms for hyerparamter}

We first target for the estimation of unknown parameter ${\mathbf{w}_{jk}^{o'}}$ and hyperparameter {${\mathbf{s}}_{jk}^{o'}$}.
In the sequel, we show that the stated program can be formulated as a \textit{convex-concave procedure} (CCCP) for ${\mathbf{w}_{jk}^{o'}}$ and {${\mathbf{s}}_{jk}^{o'}$}.
\begin{proposition}
\label{proposition:cost1}
The following programme 
$$\min_{{\mathbf{w}_{jk}^{o'}}, {{\mathbf{s}}_{jk}^{o'}}} \mathcal{L}({\mathbf{w}_{jk}^{o'}}, {{\mathbf{s}}_{jk}^{o'}})$$
with the cost function defined as
\begin{equation}
\begin{aligned}
& \mathcal{L}({\mathbf{w}_{jk}^{o'}}, {{\mathbf{s}}_{jk}^{o'}}) 
\define  {\mathbf{w}_{jk}^{o'}} {\Hessian}({\mathbf{w}_{jk}^{o'}}^*) {\mathbf{w}_{jk}^{o'}}+ 2{\mathbf{w}_{jk}^{o'}} \left[\grad({\mathbf{w}_{jk}^{o'}}^*)-{\Hessian}({\mathbf{w}_{jk}^{o'}}^*) {\mathbf{w}_{jk}^{o'}}^{*}\right] +{\mathbf{w}_{jk}^{o'}}  {{\mathbf{s}}_{jk}^{o'}}^{-1} {\mathbf{w}_{jk}^{o'}} 
\\
& +\log  |{{\mathbf{s}}_{jk}^{o'}}|+\log | {{\mathbf{s}}_{jk}^{o'}}^{-1}  +{\Hessian}({\mathbf{w}_{jk}^{o'}}^*)| 
\label{general:cost:cccp}
\end{aligned}
\end{equation}
can be formulated as a convex-concave procedure (CCCP), where ${\mathbf{w}_{jk}^{o'}}^*$ can be arbitrary real vector. 
\end{proposition}
\begin{proof}

\emph{Fact on convexity:} the function 
\begin{equation}
\begin{aligned}
{u}\left({\mathbf{w}_{jk}^{o'}}, {{\mathbf{s}}_{jk}^{o'}} \right) = & {\mathbf{w}_{jk}^{o'}} {\Hessian}({\mathbf{w}_{jk}^{o'}}^*) {\mathbf{w}_{jk}^{o'}}+ 2{\mathbf{w}_{jk}^{o'}} \left[\grad({\mathbf{w}_{jk}^{o'}}^*)-{\Hessian}({\mathbf{w}_{jk}^{o'}}^*) {\mathbf{w}_{jk}^{o'}}^{*}\right] +{\mathbf{w}_{jk}^{o'}}  {{\mathbf{s}}_{jk}^{o'}}^{-1} {\mathbf{w}_{jk}^{o'}} \\
\propto & ({\mathbf{w}_{jk}^{o'}}-{\mathbf{w}_{jk}^{o'}}^*) {\Hessian}({\mathbf{w}_{jk}^{o'}}^*) ({\mathbf{w}_{jk}^{o'}}-{\mathbf{w}_{jk}^{o'}}^*)+ 2{\mathbf{w}_{jk}^{o'}} \grad({\mathbf{w}_{jk}^{o'}}^*) +{\mathbf{w}_{jk}^{o'}}  {{\mathbf{s}}_{jk}^{o'}}^{-1} {\mathbf{w}_{jk}^{o'}}
\label{general:summary:function:u}
\end{aligned}
\end{equation}
is convex jointly in ${\mathbf{w}_{jk}^{o'}}$, ${{\mathbf{s}}_{jk}^{o'}}$ due to the fact that $f(\mathbf{x}, Y) = \mathbf{x} \mathbf{Y}^{-1} \mathbf{x}$ is jointly convex in $\bx$, $\mathbf{Y}$ (see, \citep[p.76]{boyd2004convex}). Hence $u$ as a sum of convex functions is convex. 

\emph{Fact on concavity:} 
the function 
\begin{equation}
\begin{aligned}
v({{\mathbf{s}}_{jk}^{o'}}) =   \log |{{\mathbf{s}}_{jk}^{o'}}| + \log| {{\mathbf{s}}_{jk}^{o'}}^{-1}  + {\Hessian}({\mathbf{w}_{jk}^{o'}}^*)| 
\label{general:summary:function:v}
\end{aligned}
\end{equation}
is jointly concave in ${{\mathbf{s}}_{jk}^{o'}}$, $\Cov$. We exploit the properties of the determinant of a matrix
\begin{gather*}
|A_{2 2}| |A_{1 1} - A_{1 2} A_{2 2}^{-1} A_{2 1}| = \left|\begin{pmatrix}
A_{1 1} & A_{1 2} \\
A_{2 1} & A_{2 2}
\end{pmatrix}\right| = |A_{1 1}| |A_{2 2} - A_{2 1} A_{1 1}^{-1} A_{1 2}|.
\end{gather*}
Then we have
\begin{equation}
\begin{aligned}
v({{\mathbf{s}}_{jk}^{o'}}) & =  \log |{{\mathbf{s}}_{jk}^{o'}}| + \log|  {{\mathbf{s}}_{jk}^{o'}}^{-1}  + {\Hessian}({\mathbf{w}_{jk}^{o'}}^*)|   \\
& =   \log \left(|{{\mathbf{s}}_{jk}^{o'}}||  {{\mathbf{s}}_{jk}^{o'}}^{-1}  + {\Hessian}({\mathbf{w}_{jk}^{o'}}^*)| \right)  \\
& = \log \left| \begin{pmatrix}
{\Hessian}({\mathbf{w}_{jk}^{o'}}^*) &  \\
      & -{{\mathbf{s}}_{jk}^{o'}}
\end{pmatrix} \right| \\
& = \log\left|{{\mathbf{s}}_{jk}^{o'}} +  {\Hessian}^{-1}({\mathbf{w}_{jk}^{o'}}^*)  \right| + \log \left|{\Hessian}({\mathbf{w}_{jk}^{o'}}^*) \right|
\label{}
\end{aligned}
\end{equation}
which is a $\log$-determinant of an affine function of semidefinite matrices $\Cov$, ${{\mathbf{s}}_{jk}^{o'}}$ and hence concave. 

Therefore, we can derive the iterative algorithm solving the CCCP. We have the following iterative convex optimization program by calculating the gradient of concave part.
\begin{alignat}{2}
{\mathbf{w}_{jk}^{o'}}(t)
&=\argmin\limits_{{\mathbf{w}_{jk}^{o'}}} u({\mathbf{w}_{jk}^{o'}}, {{\mathbf{s}}_{jk}^{o'}}(t-1), {\Hessian}({\mathbf{w}_{jk}^{o'}}^*)), \label{general:summary:eq:cccp1}\\
{{\mathbf{s}}_{jk}^{o'}}(t)
&=\argmin\limits_{{{\mathbf{s}}_{jk}^{o'}} \succeq \mathbf{0}} u({\mathbf{w}_{jk}^{o'}}, {{\mathbf{s}}_{jk}^{o'}}(t-1),{\Hessian}({\mathbf{w}_{jk}^{o'}}^*))+\nabla_{{{\mathbf{s}}_{jk}^{o'}}} v({{\mathbf{s}}_{jk}^{o'}}(t-1), {\Hessian}({\mathbf{w}_{jk}^{o'}}^*)){{\mathbf{s}}_{jk}^{o'}(t-1)}. \label{general:summary:eq:cccp2}
\end{alignat}
\hfill $\blacksquare$
\end{proof}

\subsubsection{Derivation of iterative reweighted $\ell_1$ algorithm}
\label{algorithms procedures formulation}
Using basic principles in convex analysis, we then obtain the following analytic form for the negative gradient of $v({{\mathbf{s}}_{jk}^{o'}})$ at ${{\mathbf{s}}_{jk}^{o'}}$ is (using chain rule):
\begin{equation}
\begin{aligned}
&\nabla_{{{\mathbf{s}}_{jk}^{o'}}} v\left({{\mathbf{s}}_{jk}^{o'}},{\Hessian}({\mathbf{w}_{jk}^{o'}}^*)\right)  |_{{{\mathbf{s}}_{jk}^{o'}}={{\mathbf{s}}_{jk}^{o'}}(t-1)}
\\
=&\nabla_{{{\mathbf{s}}_{jk}^{o'}}}\left(\log | {{\mathbf{s}}_{jk}^{o'}}^{-1}+{\Hessian}({\mathbf{w}_{jk}^{o'}}^*) |+\log  |{{\mathbf{s}}_{jk}^{o'}}| \right)|_{{{\mathbf{s}}_{jk}^{o'}}={{\mathbf{s}}_{jk}^{o'}}(t-1)}
\\
=& -\frac{(({{\mathbf{s}}_{jk}^{o'}}(t-1))^{-1}+{\Hessian}({\mathbf{w}_{jk}^{o'}}^*))^{-1}}{({{\mathbf{s}}_{jk}^{o'}}(t-1))^2} +\frac{1}{{{\mathbf{s}}_{jk}^{o'}}(t-1)}
\label{general:summary:gammastarupdate2}
\end{aligned}
\end{equation}

Combined with \eqref{general:prior:variance}, we denote a new hyper-parameter $w_{ij}^{o'}$ as following: 
\begin{equation}
\begin{aligned}
\omega_{jk}^{o'}(t) 
=& \sqrt{-\frac{(({{\mathbf{s}}_{jk}^{o'}}(t-1))^{-1}+{\Hessian}({\mathbf{w}_{jk}^{o'}}(t)))^{-1}}{({{\mathbf{s}}_{jk}^{o'}}(t-1))^2} +\frac{1}{{{\mathbf{s}}_{jk}^{o'}}(t-1)}}
\\
=& \frac{ \sqrt{{\mathbf{s}}_{jk}^{o'}(t-1)- {\rm C}_{jk}^{o'}(t)}}{{\mathbf{s}}_{jk}^{o'}(t-1)}
\label{general:summary:cccp-6}
\end{aligned}
\end{equation}
Therefore, the iterative procedures~\eqref{general:summary:eq:cccp1} and~\eqref{general:summary:eq:cccp2} for ${\mathbf{w}_{jk}^{o'}}(t)$ and ${{\mathbf{s}}_{jk}^{o'}}(t)$ can be formulated as
\begin{equation}
\begin{aligned}
& \left[{\mathbf{w}_{jk}^{o'}}(t),{{\mathbf{s}}_{jk}^{o'}}(t)\right] \\
= & \argmin \limits_{{{\mathbf{s}}_{jk}^{o'}}(t)\succeq \mathbf{0},{\mathbf{w}_{jk}^{o'}}(t)}
({\mathbf{w}_{jk}^{o'}}(t)-{\mathbf{w}_{jk}^{o'}}^*) {\Hessian}({\mathbf{w}_{jk}^{o'}}^*) ({\mathbf{w}_{jk}^{o'}}(t)-{\mathbf{w}_{jk}^{o'}}^*)
\\
& + 2{\mathbf{w}_{jk}^{o'}}(t) \grad({\mathbf{w}_{jk}^{o'}}^*) +\left(\frac{ {{\mathbf{w}_{jk}^{o'}}(t)}^2}{{{\mathbf{s}}_{jk}^{o'}}(t-1)} +{\omega_{jk}^{o'}(t)}^2{{\mathbf{s}}_{jk}^{o'}}(t-1)\right)\\
= & \argmin\limits_{{\mathbf{w}_{jk}^{o'}}(t)} {\mathbf{w}_{jk}^{o'}}(t) \Hessian({\mathbf{w}_{jk}^{o'}}^*) {\mathbf{w}_{jk}^{o'}}(t)
\\
& + 2{\mathbf{w}_{jk}^{o'}}(t) \left(\grad({\mathbf{w}_{jk}^{o'}}^*)-\Hessian({\mathbf{w}_{jk}^{o'}}^*) {\mathbf{w}_{jk}^{o'}}(t)^{*}\right) +\left(\frac{{{\mathbf{w}_{jk}^{o'}}(t)}^2}{{{\mathbf{s}}_{jk}^{o'}}(t)} +{\omega_{jk}^{o'}(t)}^2{{\mathbf{s}}_{jk}^{o'}}(t)\right).
\label{general:summary:cccp-4}
\end{aligned}
\end{equation}
Or in the compact form
\begin{equation}
\begin{aligned}
& \left[{\mathbf{w}_{jk}^{o'}}(t),{{\mathbf{s}}_{jk}^{o'}}(t)\right] \\
= &\argmin\limits_{{\mathbf{w}_{jk}^{o'}}(t)} {\mathbf{w}_{jk}^{o'}}(t) \Hessian({\mathbf{w}_{jk}^{o'}}^*) {\mathbf{w}_{jk}^{o'}}(t)+ 2{\mathbf{w}_{jk}^{o'}}(t) \left(\grad({\mathbf{w}_{jk}^{o'}}^*)-\Hessian({\mathbf{w}_{jk}^{o'}}^*) {\mathbf{w}_{jk}^{o'}}(t)^{*}\right)\\
&+{{\mathbf{w}_{jk}^{o'}}(t)}^2 {{{\mathbf{s}}_{jk}^{o'}}(t)}^{-1} + {\omega_{jk}^{o'}(t)}^2{{\mathbf{s}}_{jk}^{o'}}(t).
\label{general:summary:cccp-4.5}
\end{aligned}
\end{equation}
Since $$\frac{{{\mathbf{w}_{jk}^{o'}}(t)}^2}{{{\mathbf{s}}_{jk}^{o'}}(t)} +\omega_{jk}^{o'}(t)^2{{\mathbf{s}}_{jk}^{o'}}(t) \geq 2\left|\sqrt{{\omega_{jk}^{o'}(t)}^2} \cdot  {\mathbf{w}_{jk}^{o'}}(t)\right|,$$ the optimal ${{\mathbf{s}}_{jk}^{o'}}(t)$ can be obtained as:
\begin{equation}
\begin{aligned}
{{\mathbf{s}}_{jk}^{o'}}(t)=\frac{| {\mathbf{w}_{jk}^{o'}}(t)|}{\sqrt{{\omega_{jk}^{o'}(t)}^2}},  \forall i.
\label{general:summary:cccp-5}
\end{aligned}
\end{equation}
${\mathbf{w}_{jk}^{o'}}(t)$ can be obtained as follows
\begin{equation}
\begin{aligned}
{\mathbf{w}_{jk}^{o'}}(t)=\argmin\limits_{{\mathbf{w}_{jk}^{o'}}} &
\frac{1}{2}{\mathbf{w}_{jk}^{o'}(t)} \Hessian({\mathbf{w}_{jk}^{o'}(t)}^*) {\mathbf{w}_{jk}^{o'}(t)}+ {\mathbf{w}_{jk}^{o'}(t)} \left(\grad({\mathbf{w}_{jk}^{o'}(t)}^*)-\Hessian({\mathbf{w}_{jk}^{o'}(t)}^*) {\mathbf{w}_{jk}^{o'}(t)}^{*}\right) +\|\omega_{jk}^{o'} \cdot {\mathbf{w}_{jk}^{o'}(t)}\|_{\ell_1}.
\label{general:summary:rwglasso}
\end{aligned}
\end{equation}
We can then inject this into~\eqref{general:summary:cccp-5}, which yields
\begin{equation}
\begin{aligned}
{\mathbf{s}}_{jk}^{o'}(t) = \frac{\mathbf{w}_{jk}^{o'}(t)}{\omega_{jk}^{o'}(t)}
\label{general:summary:gammaik+1}
\end{aligned}
\end{equation}
The update rules for ${\mathbf{s}}_{jk}^{o'}$ without considering the dependency has been explained above.
However, as illustrated in Sec~.\ref{sec:performance}, the dependency between a node and its predecessors should not be disregarded. It means the dependency between edge $e_{jk}^{o'}$ and $\sum_{i<j}e_{ij}^{o}$ should be taken into consideration, then Gaussian prior could be defined as \eqref{eq:ppath} and \eqref{eq:prior_weights_2}:
$$p(\mathbf{w}\given \mathbf{\mathbf{s}}) = \prod_{j<k} \prod_{o \in \mathcal{O}} \prod_{o' \in \mathcal{O}}
p({w_{jk}^{o'}\sum_{i<j}w_{jk}^{o}}) = 
\prod_{j<k} \prod_{o \in \mathcal{O}} \prod_{o' \in \mathcal{O}} \bN(w_{jk}^{o'}\sum_{i<j}w_{ij}^{o}|0, {\mathbf{s}}_{jk}^{o'})
$$
based on this prior, the uncertainty of ${\mathbf{w}_{jk}^{o'}}(t)$ should be computed
as:
\begin{equation}
\begin{aligned}
{\gamma_{ij}^{o'}}(t) = \left(\frac{1}{\sum\limits_{i<j}\sum\limits_{o\in \mathcal{O}}{\mathbf{s}}_{ij}^{o}(t)} + \frac{1}{{\mathbf{s}}_{jk}^{o'}(t)}\right)^{-1}
\label{general:dependency gamma}
\end{aligned}
\end{equation}
As we found the expression of~\eqref{general:summary:cccp-6}, $\omega_{jk}^{o'}(t)$ is function of ${\mathbf{s}}_{jk}^{o'}(t-1)$, therefore ${\mathbf{s}}_{jk}^{o'}(t)$ is function of ${\mathbf{s}}_{jk}^{o'}(t-1)$ and ${\mathbf{w}_{jk}^{o'}}(t)$. We notice that the update for ${\mathbf{w}_{jk}^{o'}}(t)$ is to use the \emph{lasso} or \emph{$\ell_1$-regularised} regression type optimization. 
The pseudo code is summarised in Algorithm~\ref{algorithm}.

\subsection{Algorithm for Proxy Tasks}
\label{app:algorithm for proxy tasks}

\begin{algorithm}[ht]
\caption{The proposed Algorithm is transferable for cell selection on proxy tasks. 
}
\begin{algorithmic}
    \REQUIRE 
    $\boldsymbol{\gamma}_{jk}^{o'}(0), \boldsymbol{\omega}_{jk}^{o'}(0), \mathbf{w}(0) = \mathbf{1}$; sparsity intensity $\lambda_{w}^o \in \R^{+}$; $\lambda=0.01$; cost function $\mathcal{L}_D$ in \eqref{algo:eq1}
    \ENSURE 
    \FOR{$t=1$ to $T_{\max}$}
            \STATE
            1. Maximum likelihood with regularization:
            \begin{align}
            \min\limits_{\mathcal{W}, \mathbf{w}} {E_D}(\cdot) +  \lambda_w \sum_{g}\sum\limits_{j<k} \sum\limits_{o' \in \mathcal{O}}  \|{\omega_{jk,g}^{o'}(t)  \mathbf{w}_{jk,g}^{o'}}\|_2 + \lambda \|{\mathcal{W}}\|_2^2
            \end{align}
            \STATE
            2. Compute Hessian for $\mathbf{w}$ (\eqref{eq:pre_Hessian_cat}, \ref{eq:scalar_pre_Hessian}, \ref{eq:scalar_pre_Hessian_approximation})
            \STATE
            3. Update variables associated with $\mathbf{w}$
            \WHILE{$g \in (1,\Nblock); i<j<k; o, o'\in \mathcal{O}$}
                \STATE 
                \begin{align}
                \label{algo2:hyper_paras_update} 
                & C_{jk}^{o'}(t) = \left(\frac{1}{{\gamma_{jk}^{o'}}(t-1)} + {\mathbf{H}_{jk}^{o'}(t)}\right)^{-1}
                \\
                & {\omega_{jk}^{o'}}(t) \text{ is given by } \ref{general:single_group_update_omega} \text{ and } \ref{general:cell_group_update_omega}
                \\
                & 
                {{\mathbf{s}}_{jk}^{o'}}(t) \text{ is given by } \ref{general:single_group_update_local_gamma} \text{ and } \ref{general:cell_group_update_gamma}
				\\
                & {\gamma_{jk}^{o'}}(t) \text{ is given by } \ref{eq:gamma} \text{ or } \ref{eq:gamma-2},
                {\gamma}_{jk}^{o'}(t) = \left[
                \begin{array}{c|c|c}
                \myunderbrace{{\gamma}_{jk,1}^{o'}(t), \ldots, {\gamma}_{jk,1}^{o'}(t) }{\dimBr_1 ~\textbf{elements}}& \ldots & \myunderbrace{{\gamma}_{jk,\Nblock}^{o'}(t), \ldots, {\gamma}_{jk,\Nblock}^{o'}(t) }{\dimBr_{\Nblock} ~\textbf{elements}}\\
                \end{array}
                \right]
                \end{align}
		\ENDWHILE
    \STATE {4. Prune the architecture if the entropy $\frac{\ln(2\pi e {\gamma_{jk}^{o'}})}{2} \leq 0 $}
    \STATE 5. Fix $\mathbf{w} = \mathbf{1}$, train the pruned net in the standard way
    \ENDFOR
\end{algorithmic}
\label{algorithm:proxy_cell}
\end{algorithm}  

Our algorithm can be easily transferred to the scenario of proxy tasks to find the cell. Suppose a network is assembled by stacking $\Nblock$ different kinds of cells together, such as $\dimBr_1$ normal cells and $\dimBr_{\Nblock}$ reduction cells in \citep{liu2018darts}. Then optimal $\Nblock$ cells are required to be designed in a NAS task. As explained before, we design a switch ${\mathbf{s}}$ for each architecture parameter $w$ to determine the ``on-off'' of the corresponding edge in our method. In order to find such optimal cells, we propose that switches on the same position of the identical kind of cells should also be same. Based on this, the architecture parameters could be divided into different groups. The general grouped architecture parameters are given as follows:
\begin{equation}
\begin{aligned}
{\mathbf{w}_{jk}^{o'}}(t)
 =  \left[
\begin{array}{c|c|c}
\myunderbrace{{\mathbf{w}}_{jk,1}^{{o}'1}(t), \ldots, {\mathbf{w}}_{jk,1}^{o' \dimBr_1}(t) }{\dimBr_1 ~\textbf{elements}}& \ldots & \myunderbrace{{\mathbf{w}}_{jk,\Nblock}^{o'1}(t), \ldots, {\mathbf{w}}_{jk,\Nblock}^{o' \dimBr_\Nblock}(t)}{\dimBr_{\Nblock} ~\textbf{elements}}\\
\end{array}
\right].
\label{general:architecture_parameters_group}
\end{aligned}
\end{equation}
Similar to \eqref{general:summary:cccp-5}, if the group $\nblock$ is consist of $\dimBr_{\nblock}$ elements, where $\nblock = 1, \ldots, \Nblock$, the optimal $s_{jk,\nblock}^{o'}$ can be obtained as: 
\begin{equation}
\begin{aligned}
\sum_{i=1}^{\dimBr_\nblock}\left(\frac{{\mathbf{w}_{jk,\nblock}^{o'}}^\top(t) {\mathbf{w}_{jk,\nblock}^{o'}}(t)}{s_{jk,\nblock}^{o'}(t)} +{\omega_{jk,\nblock}^{o'i}(t)}^2{s_{jk,\nblock}^{o'}(t)}\right) \geq 2\left\Vert\sqrt{\sum_{i=1}^{\dimBr_\nblock} {\omega_{jk,\nblock}^{o'i}(t)}^2} \cdot {\mathbf{w}_{jk,\nblock}^{o'}(t)}\right\Vert_{\ell_2},
\end{aligned}
\end{equation}
then
\begin{equation}
\begin{aligned}
{s_{jk,\nblock}^{o'}}(t)=\frac{\left\Vert {\mathbf{w}_{jk,\nblock}^{o'}(t)}\right\Vert_{\ell_2}}{\sqrt{\sum_{i=1}^{\dimBr_\nblock} {{\omega_{jk,\nblock}^{o'i}(t)}}^2}},  \forall i.
\label{general:single_group_update_local_gamma}
\end{aligned}
\end{equation}
The calculation of $\omega_{jk,g}^{o'}$ for group $\nblock$ is:
\begin{equation}
\begin{aligned}
{\omega}_{jk,o}^{o'}(t) = \sqrt{ \sum_{i=1}^{\dimBr_\nblock} \frac{ \sqrt{{\gamma}_{jk,o}^{o'i}(t-1)- {\rm C}_{jk,o}^{o'i}(t)}}{{{{\gamma}_{jk,o}^{o'i}(t-1)}}^2}}
\label{general:single_group_update_omega}
\end{aligned}
\end{equation}
and both ${\mathbf{s}}$ and $\omega$ for the different elements in identity group should keep the same:
\begin{equation}
\begin{aligned}
 {\mathbf{s}}_{jk}^{o'}(t) = \left[
    \begin{array}{c|c|c}
    \myunderbrace{{\mathbf{s}}_{jk,1}^{o'}(t), \ldots, {\mathbf{s}}_{jk,1}^{o'}(t) }{\dimBr_1 ~\textbf{elements}}& \ldots & \myunderbrace{{\mathbf{s}}_{jk,\Nblock}^{o'}(t), \ldots, {\mathbf{s}}_{jk,\Nblock}^{o'}(t) }{\dimBr_{\Nblock} ~\textbf{elements}}\\
    \end{array}
    \right]
    \label{general:cell_group_update_gamma}
\end{aligned}
\end{equation}
\begin{equation}
\begin{aligned}
 {\omega}_{jk}^{o'}(t) = \left[
    \begin{array}{c|c|c}
    \myunderbrace{{\omega}_{jk,1}^{o'}(t), \ldots, {\omega}_{jk,1}^{o'}(t) }{\dimBr_1 ~\textbf{elements}}& \ldots & \myunderbrace{{\omega}_{jk,\Nblock}^{o'}(t), \ldots, {\omega}_{jk,\Nblock}^{o'}(t) }{\dimBr_{\Nblock} ~\textbf{elements}}\\ \\
    \end{array}
    \right]
    \label{general:cell_group_update_omega}
\end{aligned}
\end{equation}
It should be noted that the detailed derivation procedures can be referred to \ref{general:prior:sec:costfunction} and \ref{algrithm procedures}.
The pseudo code is summarised in Algorithm~\ref{algorithm:proxy_cell}.

\section{Structural Bayesian Deep Compression}
\label{app:structural compression algorithm}
\begin{algorithm}[ht]
\caption{
The proposed Algorithm is transferable for neural network compression.}
\begin{algorithmic}
    \REQUIRE 
    Initialization: $\forall l = 1,\ldots,L,$  ${\omega}^l(0), {\gamma}^l(0) = 1$; $\lambda^l \in \R^{+}$;
    \ENSURE 
    \FOR{$t=1$ to $T_{\max}$}
        \STATE 1. Maximum likelihood with regularization:
            \STATE
            \begin{align}
            \label{algo3:eq1}
            \min\limits_{{\mathcal{W}}} {E_D}(\cdot) +  \sum\limits_{l=1}^{L} \lambda^l R(\omega^l(t) \circ {\mathcal{W}}^l)
            \end{align}
        \STATE 2. Compute the Hessian for fully connected layer and convolutional layers as Appendix.~\ref{appendix:fchessian} and \ref{appendix:convhessian} respectively.
        \STATE 3. Update hyper-parameters:
        \\
        \COMMENT{\text{{\rm Update}() specifies how to update parameters, detailed update rules are in Table~\ref{tab:cnn_mlp}}}
        \STATE 
            \begin{align}
        \label{algo3:hyper_paras_update} 
            & {\gamma}^l(t) \leftarrow  {\rm Update}(\omega^l(t-1), {\mathcal{W}^l}(t)), {\Gamma^l}(t) = [{\gamma}^l(t)]
            \\
            & C^l(t) \leftarrow \left(({{\Gamma}^l}(t))^{-1} + \diag({\rm H}^l(t))\right)^{-1},
            \\
            & {\alpha}^l(t) \leftarrow   -{\frac{{\rm C}^l(t)}{{{\gamma}^l(t)}^2}}+\frac{1}{{\gamma}^l(t)} \COMMENT{\text{element-wise division}}
            \\
            & 
            {\omega}^l(t) \leftarrow  {\rm Update}({\alpha}^{l}(t))
            \end{align}
        \STATE {4. Prune the unimportant connections.}
    \ENDFOR
\end{algorithmic}
\label{algorithm:structural_sparsity}
\end{algorithm}

In addition to applying the proposed Bayesian approach to address NAS problem, we also explore the possibility of our method on network structural compression problem. In this section, we extend to compress deep neural networks by proposing a series of generic and easily implemented reweighted group Lasso algorithms to solve maximization of marginal likelihood $\int p(\mathbf{Y}\given\boldsymbol{\mathcal{W}})p(\boldsymbol{\mathcal{W}})d\boldsymbol{\mathcal{W}}$ where $p(\boldsymbol{\mathcal{W}})$ can be specified as various sparse structured priors over network weights as shown in Table~\ref{tab:cnn_mlp}. The proposed Algorithm is generic for the weights in fully connected and convolutional neural networks. The training algorithm is iteratively indexed by $t$. Each iteration contains several epochs. Within each iteration $t$, there are three parts, The first part is simply a reweighted group Lasso type optimization. In the regularization terms $R(\omega^l \circ {\mathcal{W}^l})$, each weight of layer $l$ is scaled by a factor $(\omega^l)$. The update of $(\omega^l)$ needs the Hessian of each $\boldsymbol{\mathcal{W}}^l$. The Hessian of each layer can be computed recursively given the Hessian in the next layer through backward passing. The hyper-parameters $\gamma^l$, $C^l$ and $\omega^l$ will be updated every iteration $t$ with $T_{\max}$ being the maximal iterations. $\alpha^l$ is an introduced intermediate variable during the update process. The pseudo code is summarized in Algorithm~\ref{algorithm:structural_sparsity}.

\begin{figure}[]
\centering
\includegraphics[scale=0.5]{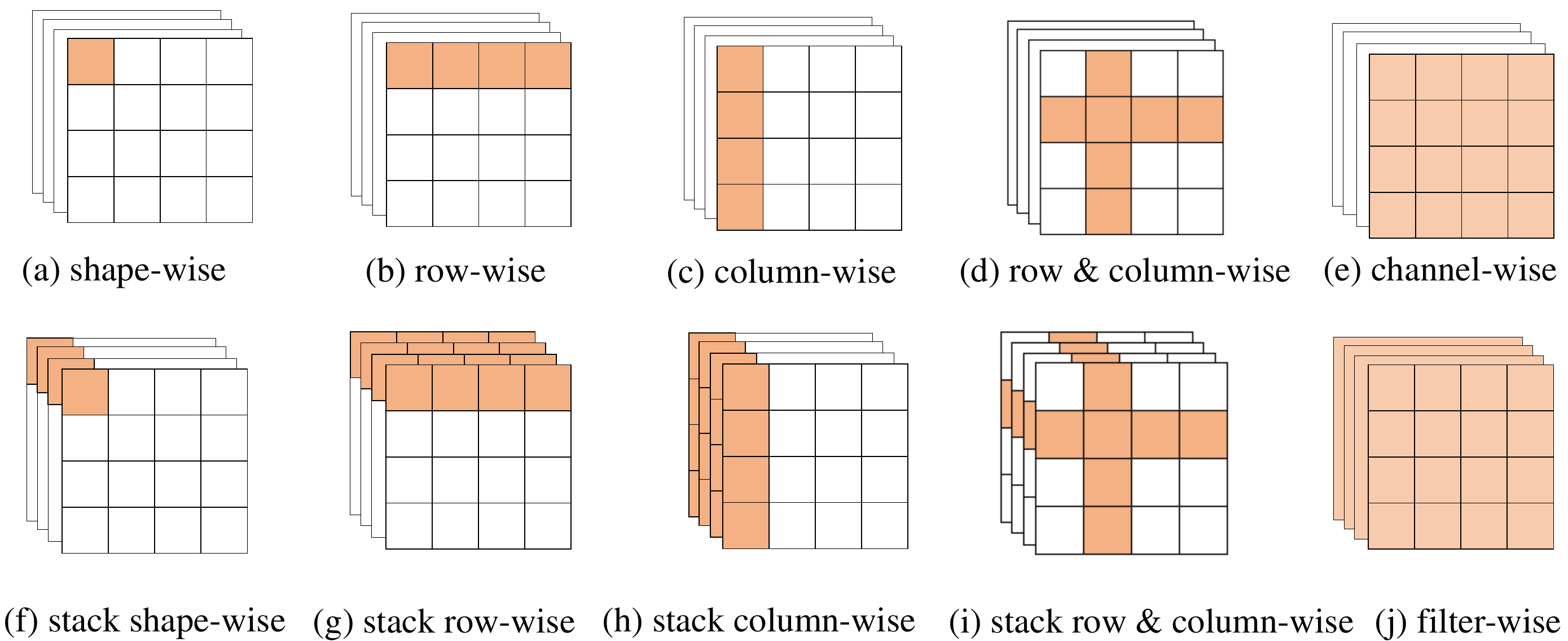}
\caption{some examples of structured sparsity for the 3D filters in Conv layer with extensions of \cite{wen2016learning}. Coloured squares mean the weights to be pruned. It should be noted that the FC layer can be easily enforced by (a)-(e).}
\label{fig:cnn_shape}
\end{figure}

\begin{table*}[ht]
\centering
\caption{Hyper-parameter update rule in Algorithm\ref{algorithm:structural_sparsity} for CNN}
\label{tab:cnn_mlp}
\resizebox{1\textwidth}{!}{
\begin{tabular}{|c|c|c|c|c|}
\hline
Category  & sparse prior     & $R^l(\omega\circ\boldsymbol{\mathcal{W}})$  & ${\omega}^{l}$    & ${\gamma}^{l}$  \\ \hline
(a) Shape-wise & $\prod\limits_{c_l}\prod\limits_{m_l}\prod\limits_{k_l}\mathcal{N}(\mathbf{0}, \gamma_{c_l,m_l,k_l}\mathbf{I}_{n_l}) $               
& $\sum\limits_{c_l = 1}^{C_l}\sum\limits_{m_l = 1}^{M_l}\sum\limits_{k_l = 1}^{K_l}\|\omega_{:,{c_l},{m_l},{k_l}}^l \circ \boldsymbol{\mathcal{W}}_{:,{c_l},{m_l},{k_l}}^l\|_2$                    
& \begin{tabular}[c]{@{}c@{}}
    ${\omega}^{l}_o = \sqrt{\sum\limits_{c_l}\sum\limits_{m_l}\sum\limits_{k_l}|\alpha_{:,{c_l},{m_l},{k_l}}^l|}$\\ 
    ${\omega}_{:,{c_l},{m_l},{k_l}}^l = {\omega}^{l}_o \cdot \rmI_{:,{c_l},{m_l},{k_l}}^l$
\end{tabular}                                                                       
& \begin{tabular}[c]{@{}c@{}}
    ${\gamma}^{l}_o = \frac{\|\boldsymbol{\mathcal{W}}_{:,{c_l},{m_l},{k_l}}^l\|_2}{\omega_{:,{c_l},{m_l},{k_l}}^{l}(t-1)}$\\ 
    $\gamma_{:,{c_l},{m_l},{k_l}}^l = {\gamma}^{l}_o \cdot \rmI_{:,{c_l},{m_l},{k_l}}^l$
\end{tabular}  \\ \hline
(b) Row-wise  & $\prod\limits_{c_l}\prod\limits_{m_l}\mathcal{N}(\mathbf{0}, \gamma_{c_l,m_l}\mathbf{I}_{n_lk_l}) $ 
& $\sum\limits_{c_l = 1}^{C_l}\sum\limits_{m_l = 1}^{M_l}\|\omega_{:,{c_l},{m_l},:}\circ \boldsymbol{\mathcal{W}}_{:,{c_l},{m_l},:}^l\|_2$       
& \begin{tabular}[c]{@{}c@{}}
    ${\omega}^{l}_o = \sqrt{\sum\limits_{c_l}\sum\limits_{m_l}{|\alpha_{:,{c_l},{m_l},:}^l|}}$\\ 
    ${\omega}_{:,{c_l},{m_l},:}^l = {\omega}^{l}_o \cdot \rmI_{:, {c_l},{m_l},:}^l
$\end{tabular}                                                                              
& \begin{tabular}[c]{@{}c@{}}
    ${\gamma}^{l}_o = \frac{\|\boldsymbol{\mathcal{W}}_{:,{c_l},{m_l},:}^l\|_2}{\omega_{:,{c_l},{m_l},:}^{l}(t-1)}$\\ 
    $\gamma_{:,{c_l},{m_l},:}^l = {\gamma}^{l}_o \cdot \rmI_{:,{c_l},{m_l},:}^l$
\end{tabular}                                                                  
\\ \hline
(c) Column-wise &    $\prod\limits_{c_l}\prod\limits_{k_l}\mathcal{N}(\mathbf{0}, \gamma_{c_l,k_l}\mathbf{I}_{n_lm_l}) $            
& $\sum\limits_{c_l = 1}^{C_l}\sum\limits_{k_l = 1}^{K_l}\|\omega_{:,{c_l},:,{k_l}}^l \circ \boldsymbol{\mathcal{W}}_{:,{c_l},:,{k_l}}^l\|_2$                                                                     
& \begin{tabular}[c]{@{}c@{}}
    ${\omega}^{l}_o = \sqrt{\sum\limits_{c_l}\sum\limits_{k_l}{|\alpha_{:,{c_l},:,{k_l}}^l|}}$\\ 
    ${\omega}_{:,{c_l},:,{k_l}}^l = {\omega}^{l}_o \cdot \rmI_{:, {c_l},:,{k_l}}^l$
\end{tabular}                                                                              
& \begin{tabular}[c]{@{}c@{}}
    ${\gamma}^{l}_o = \frac{\|\boldsymbol{\mathcal{W}}_{:,{c_l},:,{k_l}}^l\|_2}{\omega_{:,{c_l},:,{k_l}}^{l}(t-1)}$\\ 
    $\gamma_{:,{c_l},:,{k_l}}^l = {\gamma}^{l}_o \cdot \rmI_{:,{c_l},:,{k_l}}^l$
\end{tabular}                                                                  \\ \hline
(d) Row $\&$ column-wise   &  $\prod\limits_{c_l}\prod\limits_{m_lk_l}\mathcal{N}(\mathbf{0}, \gamma_{c_l, m_lk_l}\mathbf{I}_{n_l}) $   
& \begin{tabular}[c]{@{}c@{}}
    ${\bar{W}}^l = [\boldsymbol{\mathcal{W}}_{:,{c_l},{m_l},:}^l , \boldsymbol{\mathcal{W}}_{:,{c_l},:,{k_l}}^l]$\\ 
    ${\bar{\omega}}^{l}= [\omega_{:,{c_l},{m_l},:}^l , \omega_{:,{c_l},:,{k_l}}^l]$\\
    $\sum\limits_{c_l = 1}^{C_l}\sum\|{\bar{\omega}}^{l} \circ {\bar{W}}^l\|_2$ 
\end{tabular} 
& \begin{tabular}[c]{@{}c@{}}
    ${\omega}_o^{l}= \sqrt{\sum\limits_{c_l}\sum\limits_{m_l}{|\alpha_{:,{c_l},{m_l},:}^l|}+\sum\limits_{c_l}\sum\limits_{k_l}{|\alpha_{:,{c_l},:,{k_l}}^l|}}$\\ 
    ${\omega}_{:,{c_l},{m_l},:}^l = {\omega}^{l}_o \cdot \rmI_{:,{c_l},{m_l},:}^l$\\ 
    ${\omega}_{:,{c_l},:,{k_l}}^l = {\omega}^{l}_o \cdot \rmI_{:,{c_l},:,{k_l}}^l$
\end{tabular} 
& \begin{tabular}[c]{@{}c@{}}
    ${\gamma}^{l}_o = \frac{\|{\bar{W}}^l\|_2}{{\bar{\omega}}^{l}(t-1)}$\\ 
    $\gamma_{:,{c_l},{m_l},:}^l = {\gamma}^{l}_o \cdot \rmI_{:,{c_l},{m_l},:}^l$\\ 
    $\gamma_{:,{c_l},:,{k_l}}^l = {\gamma}^{l}_o \cdot \rmI_{:,{c_l},:,{k_l}}^l$
\end{tabular} \\ \hline
(e) Channel-wise     &   $\prod\limits_{c_l}\mathcal{N}(\mathbf{0}, \gamma_{c_l}\mathbf{I}_{n_lm_lk_l}) $       
& $\sum\limits_{c_l = 1}^{C_l}\|\omega_{:,{c_l},:,:}^l \circ \boldsymbol{\mathcal{W}}_{:,{c_l},:,:}^l\|_2$                 
& \begin{tabular}[c]{@{}c@{}}
    ${\omega}^{l}_o = \sqrt{\sum\limits_{c_l}{|\alpha_{:,{c_l},:,:}^l|}}$\\ 
    $\omega_{:,{c_l},:,:}^l = {\omega}^{l}_o \cdot \rmI_{:,{c_l},:,:}^l
$\end{tabular}                                                                                              
& \begin{tabular}[c]{@{}c@{}}
    ${\gamma}^{l}_o = \frac{\|\boldsymbol{\mathcal{W}}_{:,{c_l},:,:}^l\|_2}{\omega_{:,{c_l},:,:}^{l}(t-1)}$\\ 
    $\gamma_{:,{c_l},:,:}^l = {\gamma}^{l}_o \cdot \rmI_{:,{c_l},:,:}^l$
\end{tabular}     \\ \hline
(f) Group shape-wise & $\prod\limits_{m_l}\prod\limits_{k_l}\mathcal{N}(\mathbf{0}, \gamma_{m_l,k_l}\mathbf{I}_{n_lc_l}) $          
& $\sum\limits_{m_l = 1}^{M_l}\sum\limits_{k_l = 1}^{K_l}\|\omega_{:, :,{m_l},{k_l}}^l \circ \boldsymbol{\mathcal{W}}_{:, :,{m_l},{k_l}}^l\|_2$
& \begin{tabular}[c]{@{}c@{}}${\omega}^{l}_o = \sqrt{\sum\limits_{m_l}\sum\limits_{k_l}|\alpha_{:, :,{m_l},{k_l}}^l|}$\\ ${\omega}_{:, :, {m_l},{k_l}}^l = {\omega}^{l}_o \cdot \rmI_{:, :,{m_l},{k_l}}^l$
\end{tabular}                                                                             
& \begin{tabular}[c]{@{}c@{}}
    ${\gamma}^{l}_o = \frac{\|\boldsymbol{\mathcal{W}}_{:, :,{m_l},{k_l}}^l\|_2}{\omega_{:, :,{m_l},{k_l}}^{l}(t-1)}$\\ 
    $\gamma_{:, :,{m_l},{k_l}}^l = {\gamma}^{l}_o \cdot \rmI_{:, :,{m_l},{k_l}}^l$
\end{tabular}                                                              
\\ \hline
(g) Group row-wise &   $\prod\limits_{m_l}\mathcal{N}(\mathbf{0}, \gamma_{m_l}\mathbf{I}_{n_lc_lk_l}) $          
& $\sum\limits_{m_l = 1}^{M_l}\|\omega_{:, :,{m_l},:}^l \circ \boldsymbol{\mathcal{W}}_{:, :,{m_l},:}^l\|_2$                                                                   
& \begin{tabular}[c]{@{}c@{}}
    ${\omega}^{l}_o = \sqrt{\sum\limits_{m_l}{|\alpha_{:, :,{m_l},:}^l|}}$\\ 
    ${\omega}_{:, :,{m_l},:}^l = {\omega}^{l}_o \cdot \rmI_{:, :,{m_l},:}^l$
\end{tabular}                                                                                       
& \begin{tabular}[c]{@{}c@{}}
    ${\gamma}^{l}_o = \frac{\|\boldsymbol{\mathcal{W}}_{:, :,{m_l},:}^l\|_2}{\omega_{:, :,{m_l},:}^{l}(t-1)}$\\ 
    $\gamma_{:, :,{m_l},:}^l = {\gamma}^{l}_o \cdot \rmI_{:, :,{m_l},:}^l$
\end{tabular}     
\\ \hline
(h) Group column-wise &   $\prod\limits_{k_l}\mathcal{N}(\mathbf{0}, \gamma_{k_l}\mathbf{I}_{n_lc_lm_l}) $       
& $\sum\limits_{k_l = 1}^{K_l}\|\omega_{:, :, :,{k_l}}^l \circ \boldsymbol{\mathcal{W}}_{:, :,:,{k_l}}^l\|_2$                                                           
& \begin{tabular}[c]{@{}c@{}}
    ${\omega}^{l}_o = \sqrt{\sum\limits_{k_l}{|\alpha_{:, :, :,{k_l}}^l|}}$\\ 
    ${\omega}_{:, :, :,{k_l}}^l = {\omega}^{l}_o \cdot \rmI_{:, :, :,{k_l}}^l$
\end{tabular}                                                                                    
& \begin{tabular}[c]{@{}c@{}}
    ${\gamma}^{l}_o = \frac{\|\boldsymbol{\mathcal{W}}_{:, :, :,{k_l}}^l\|_2}{\omega_{:, :,:,{k_l}}^{l}(t-1)}$\\ 
    $\gamma_{:, :, :,{k_l}}^l = {\gamma}^{l}_o \cdot \rmI_{:, :, :,{k_l}}^l$
\end{tabular}                                            
\\ \hline
(i) Group row $\&$ column-wise &$\prod\limits_{m_lk_l}\mathcal{N}(\mathbf{0}, \gamma_{m_lk_l}\mathbf{I}_{n_lc_l}) $ 
& \begin{tabular}[c]{@{}c@{}}
    ${\bar{W}}^l = [\boldsymbol{\mathcal{W}}_{:, :,{m_l},:}^l , \boldsymbol{\mathcal{W}}_{:, :,:,{k_l}}^l]$\\ 
    ${\bar{\omega}}^{l}= [\omega_{:, :, {m_l}, :}^l , \omega_{:, :, :,{k_l}}^l]$ \\
    $\sum\|{\bar{\omega}}^{l} \circ {\bar{W}}^l\|_2$
\end{tabular}                              
& \begin{tabular}[c]{@{}c@{}}
    ${\omega}_o^{l}= \sqrt{\sum\limits_{m_l}{|\alpha_{:, :, {m_l},:}^l|}+\sum\limits_{k_l}{|\alpha_{:, :, :,{k_l}}^l|}}$\\ 
    ${\omega}_{:, :,{m_l},:}^l = {\omega}^{l}_o \cdot \rmI_{:, :,{m_l},:}^l$\\ 
    ${\omega}_{:, :, :,{k_l}}^l = {\omega}^{l}_o \cdot \rmI_{:,:,:,{k_l}}^l$
\end{tabular}           
& \begin{tabular}[c]{@{}c@{}}
    ${\gamma}^{l}_o = \frac{\|{\bar{W}}^l\|_2}{{\bar{\omega}}^{l}(t-1)}$\\ 
    $\gamma_{:, :,{m_l},:}^l = {\gamma}^{l}_o \cdot \rmI_{:, :,{m_l},:}^l$\\ 
    $\gamma_{:, :,:,{k_l}}^l = {\gamma}^{l}_o \cdot \rmI_{:, :,:,{k_l}}^l$
\end{tabular}             
\\ \hline
(j) Filter-wise      &   $\prod\limits_{n_l}\mathcal{N}(\mathbf{0}, \gamma_{n_l}\mathbf{I}_{c_lm_lk_l}) $        
& $\sum\limits_{n_l = 1}^{N_l}\|\omega_{{n_l},:,:,:}^l \circ \boldsymbol{\mathcal{W}}_{{n_l},:,:,:}^l\|_2$  
& \begin{tabular}[c]{@{}c@{}}
    ${\omega}^{l}_o = \sqrt{\sum\limits_{n_l}{|\alpha_{{n_l},:,:,:}^l|}}$\\ 
    $\omega_{{n_l},:,:,:}^l = {\omega}^{l}_o \cdot \rmI_{{n_l},:,:,:}^l$
\end{tabular}                 
& \begin{tabular}[c]{@{}c@{}}
    ${\gamma}^{l}_o =\frac{\|\boldsymbol{\mathcal{W}}_{{n_l},:,:,:}^l\|_2}{\omega_{{n_l},:,:,:}^{l}(t-1)}$ \\ 
    $\gamma_{{n_l},:,:,:}^l = {\gamma}^{l}_o \cdot \rmI_{{n_l},:,:,:}^l$
\end{tabular} \\ \hline
\end{tabular}
}
\end{table*}

\vspace{3cm}
\subsection{Structured Sparse Prior for Fully Connected and Convolutional Networks}
\label{sec:method:cnn}
For weight in the $l$-th convolutional layer $W^l \in {\R^{N_l \times C_{l} \times m_{l} \times k_{l}}}$, some examples of the structured sparsity are shown in Fig.\ref{fig:cnn_shape}. The corresponding sparse prior is given in the second column of Table~\ref{tab:cnn_mlp}. It should be noted that the prior for the weight in fc layer could also be represented as this table with $m_{l} = k_{l} =1$, $N_l$ and $C_l$ stands for the size of input feature and output feature respectively.

\subsection{Experiments}
\label{app:structural compression experiment}
\subsubsection{LeNet-300-100 and LeNet-5 on MNIST}
\label{exp:mnist}
We first perform LeNet-300-100 and LeNet-5 on MNIST dataset \cite{lecun1998mnist}.
For LeNet-300-100, we apply shape-wise, row-wise and column-wise regularization as shown in Fig.~\ref{fig:cnn_shape}(a), \ref{fig:cnn_shape}(b) and \ref{fig:cnn_shape}(c), for the 2D weight matrices.
The hyper-parameters $\gamma$, $\omega$ and $\alpha$ are updated every ten epochs for a total of $T_{\max} = 10$ loops. The learned structure is $465-37-90$ with $1.54\%$ test error and $0.04$ FLOPS \cite{molchanov2016pruning}. Comparison with other methods can be found in Table~\ref{table:lenet300100}.
For LeNet-5, we apply shape-wise and filter-wise regularization for the conv layer as shown in Fig.~\ref{fig:cnn_shape}(a) and \ref{fig:cnn_shape}(j);  row-wise and column-wise regularization for fc layer as shown in Fig.~\ref{fig:cnn_shape}(b) and \ref{fig:cnn_shape}(c).
The learned structure is $5-10-65-25$ with $1.00\%$ test error and $0.57$ FLOPS. Comparison with other methods can be found in Table~\ref{tab:lenet_5_mnist}.

\begin{table}[ht]
\caption{Comparison of the learned architecture with other methods using LeNet-300-100 on MNIST dataset}
\centering
\begin{tabular}{ccccc}
\hline
Method                          & Pruned Architecture & Error Rate (\%) & FLOPs (M) \\ \hline
Baseline                        & 784-300-100         & 1.39            & 0.53       \\
SBP (\cite{neklyudov2017structured})                         & 245-160-55          & 1.60            & 0.10      \\
BC-GNJ (\cite{louizos2017bayesian})                      &  278-98-13           & 1.80            & 0.06      \\
BC-GHS (\cite{louizos2017bayesian})                      & 311-86-14           & 1.80            & 0.06      \\
Practical $\ell_0$(\cite{louizos2017learning}) & 219-214-100         & 1.40            & 0.14      \\
Practical $\ell_0$ (\cite{louizos2017learning}) & 266-88-33           & 1.80            & 0.05      \\ \hline
Proposed method             & 465-37-90           & 1.54            & 0.04      \\ \hline
\end{tabular}
\label{table:lenet300100}
\end{table}

\begin{table}
\caption{Comparison of the learned architecture with other methods using LeNet-5 on MNIST dataset}
\label{tab:lenet_5_mnist}
\centering
\begin{tabular}{ccccc}
\hline
Method                          & Pruned Architecture & Error Rate (\%) & FLOPs (M) \\ \hline
Baseline                        & 20-50-800-500       & 0.83            & 8.85      \\
SBP (\cite{neklyudov2017structured})                         & 3-18-284-283        & 0.90            & 0.69      \\
BC-GNJ (\cite{louizos2017bayesian})                      & 8-13-88-13          & 1.00            & 1.09      \\
BC-GHS (\cite{louizos2017bayesian})                      & 5-10-76-16          & 1.00            & 0.57      \\
Practical $\ell_0$(\cite{louizos2017learning}) & 20-25-45-462        & 0.90            & 4.49      \\
Practical $\ell_0$ (\cite{louizos2017learning}) & 9-18-65-25          & 1.00            & 1.55      \\ \hline
Proposed method             & 5-10-65-25          & 1.00            & 0.57      \\ \hline
\end{tabular}
\end{table}

\subsubsection{ResNet-18 on CIFAR-10}
\label{exp:cifar10}
We also evaluate our algorithm on Cifar10 dataset using ResNet-18 as initialized backbone \cite{he2016deep}. In addition to the input conv layer and output fc layer, the other 16 conv layers are separated into 8 blocks with 2 layers each. We apply shape-wise and filter-wise regularization to the conv layer as shown in Fig.~\ref{fig:cnn_shape}(a) and \ref{fig:cnn_shape}(j);  row-wise and column-wise regularization to the fc layer as shown in Fig.~\ref{fig:cnn_shape}(b) and \ref{fig:cnn_shape}(c).
The result is given in Table~\ref{table:res18}. It can be found that two Conv layers in block 4 are pruned away which shows the potential of our method to reduce the number of layers.
\begin{table}[ht]
\centering
\caption{Sparsity for each layer in ResNet-18 on Cifar10 dataset}
\begin{tabular}{cccccc|c|c}
\hline
conv1  & conv2-5 & conv6-9 & conv10-13& con14-17 & FC layer& Total & Test error \\ \hline
22.80\% & \begin{tabular}[c]{@{}c@{}}
10.06\%\\ 22.87\%\\ 20.05\%\\ 12.99\%
\end{tabular} & \begin{tabular}[c]{@{}c@{}}9.24\%\\ 5.45\%\\ 4.94\%\\ 10.73\%\end{tabular} & \begin{tabular}[c]{@{}c@{}}20.64\%\\ 15.04\%\\ 10.77\%\\ 4.61\%\end{tabular} & \begin{tabular}[c]{@{}c@{}}1.43\%\\ 0.19\%\\ 0\\ 0\end{tabular} & 10.33\%   & 2.96\% & \begin{tabular}[c]{@{}c@{}}
6.58\% \\(\text{baseline})\\  6.23\%\\ (\text{our method})
\end{tabular} \\ \hline
\end{tabular}
\label{table:res18}
\end{table}

\section{Efficient Hessian Computation}
\label{appdendix:hessian calculation methods}
\subsection{Compute the Hessian of FC Layer}
\label{appendix:fchessian}
The mathematical operation in a fully-connected layer could be formulated as: 
\begin{align}
\label{eq:fc_feedforward} 
    h_j^o = W_{ij}^o a_i,  \ \ \ 
    a_i = \sigma(h_i) \ \ \
\end{align}

where $h_{i}$ is the pre-activation value for node $i$ and $a_{i}$ is the activation value. $\sigma()$ is the element-wise activation function. ${\mathcal{W}_{ij}^o}$ stands for the weight matrix associated with operation $o$ in edge $e_{ij}^o$. In \cite{botev2017practical}, a recursive method is proposed to compute the Hessian $\rmH$ for $\mathcal{W}_{ij}^o$:
\begin{equation}
    \label{eq:fc_Hessian}  
    {\rmH}_{ij}^o = a_i \cdot({a_i})^\top \otimes H_j^o
\end{equation}
where $\otimes$ stands for Kronecker product; The pre-activation Hessian $H_j^o$ is known and could be used to compute the pre-activation Hessian recursively for the previous layer:

\begin{align}
\label{eq:fc_pre_Hessian} 
    H_i = B_i ({\mathcal{W}_{ij}^o})^{\top} H_{j}^o \mathcal{W}_{ij}^o B_i + D_i,  \ \ \ 
    B_i = \diag(\sigma'(h_i)), \ \ \
    D_i = \diag(\sigma''(h_i) \circ \frac{{\partial}L}{{\partial}a_i})
\end{align}
 
In order to reduce computation complexity, the original pre-activation Hessian $H$ and Hessian $\rmH$ in Eq ~\ref{eq:fc_Hessian}-\ref{eq:fc_pre_Hessian} are replaced with their diagonal values for recursive computation. Thus the matrix multiplication could be reduced to vector multiplication. The hessian calculation process could be reformulated as:
\begin{equation}
    \label{eq:fc_Hessian_approximation}  
    {\rmH}_{ij}^o = a_i^2 \otimes H_j^o
\end{equation}
\begin{align}
\label{eq:fc_pre_Hessian_approximation} 
        H_i = B_i^2 \circ ((({\mathcal{W}_{ij}^o)^{\top}})^2 H_{j}^o)+ D_i,\ \ \
        B_i = \sigma'(H_i), \ \ \ D_i=\sigma''(H_i) \circ \frac{{\partial}L}{{\partial}a_i}
\end{align}
Where $\text{diag}()$ means the operation to extract the diagonal values of input variable. If we compute Hessian with the approximate method as Eq~ \ref{eq:fc_Hessian_approximation}-\ref{eq:fc_pre_Hessian_approximation}, the multiply accumulate operation (MACs) for the pre-activatiion Hessian $H$ and Hessian $\rmH$ could be reduced from $n(2m^2+2n^2+4mn+3m-1)$ to $n(2+4m)$ with $W \in \R^{n \times m}$. (e.g. with $n=100, m=100$, the original method requires $107.97$ MMACs compared with only $0.04$ MMACs for the approximate method.)

\subsection{Compute the Hessian of Conv Layer}
\label{appendix:convhessian}
Although the Hessian of weight matrix has been widely used in second-order optimization techniques to speed up the training process \citep{lecun1990optimal, amari1998natural}, it still remains infeasible to calculate explicit Hessian directly due to the intensive computation burden \citep{martens2015optimizing, botev2017practical}. 
Moreover, as most of current deep neural networks include plenty of Convolutional (Conv) layers, it further increases the difficulty of calculation due to the indirect convolution operation. Inspired by the Hessian calculation methods for Fully Connected (FC) layers as shown in \cite{botev2017practical}, we propose a recursive and efficient method to compute the Hessian of Conv layers by converting Conv layers to FC layers \cite{ma2017equivalence}. Therefore Hessian of the resulting equivalent FC layer is ready to be obtained. The detailed calculation procedures are explained in the following:

Suppose a convolution operation $o$ is selected between node $i$ and $j$ ($i<j$). The corresponding input vector, weight 
and output vector of this edge are denoted as $B_i \in \R^{b \times C_{i} \times H_i \times W_i}$, $\mathcal{W}_{ij}^o \in \R^{C_{j}^o \times C_{i} \times m_{ij}^o \times k_{ij}^o}$ and $B_j^o \in \R^{b \times C_{j}^o \times H_{j}^o \times W_{j}^o}$ respectively, where $b$ is the batch size, $C_i$, $H_i$, $W_i$ are the size of input channel, height and width; $C_j^o$ is the size of output channel, $ m_{ij}^o \times k_{ij}^o$ is the kernel dimension; $H_j^o$ and $W_{j}^o$ are the size of output height and width. 

As in \cite{ma2017equivalence}, $B_{i}$ is converted to two dimensional matrix for FC layer, with dimension ${(b H_{j}^o W_{j}^o) \times (C_{i} m_{ij}^o k_{ij}^o)}$. Similarly, the dimension of $B_{j}^o$ is changed from $b \times C_{j}^o \times H_{j}^o \times W_{j}^o$ to $(b H_{j}^o W_{j}^o) \times C_{j}^o$. The dimension of $\mathcal{W}_{ij}^{o}$ is changed to $\R^{(C_{i} m_{ij}^o k_{ij}^o ) \times C_{j}^o}$. The input vector, output vector and weight for the FC layer are denoted 
as  $M_{i}$, $M_{j}^o$ and $\mathcal{W}_{ij}^{oM}$. 

Secondly, $M_i$, $M_{j}^o$ and $H^{oM}_{j}$ are decomposed into a total of $b H_j^o W_j^o$ row vectors with $(M_i)^{n} \in \R^{C_i  m_{ij}^o k_{ij}^o}$,  $(M_j^o)^n\in \R^{C_j^o}$ and $(H^{oM}_{j})^n\in \R^{C_j^o}$ ($n = 1, \ldots, b H_j^o W_j^o$) respectively. It is easy to understand that $(M_i)^n$, $(M_{j}^o)^{n}$ could be regarded as the input vector and output vector of a FC layer with weight matrix $\mathcal{W}_{ij}^{oM}$. 
Then we can obtain the Hessian $\mathbf{H}_{ij}^o$ for $\mathcal{W}_{ij}^o$ as follows:
\begin{equation}
    ({\mathbf{H}}^{oM}_{ij})^n =(M_{i})^n \cdot {(M_{i})^n}^{\top} \otimes ({H^{oM}_{j}})^n
    \label{eq:conv_Hessian_vector}
\end{equation}
$({H^{oM}_{j}})^n$ is the pre-activation Hessian which could be computed recursively. With $({H^{oM}_{j}})^n$ known, the pre-activation Hessian for $(M_i)^n$  could be calculated as:
\begin{align}
    ({H^{M}_{i}})^n &= (B_i)^n {\mathcal{W}^{oM}_{ij}}^{\top} ({H^{oM}_{j}})^n {\mathcal{W}^{oM}_{ij}} (B_i)^n + (D_i)^n \nonumber\\
    (B_i)^n &= \diag\left(\sigma'((h_{i})^{n})\right) \nonumber \\
    (D_i)^n &= \diag\left(\sigma''((h_{i})^{n}) \frac{{\partial}L}{{\partial}(M_{i})^{n}}\right) \nonumber
\end{align}
where $(h_{i})^{n}$ is the pre-activation value for FC layer and $L$ means the loss function. The pre-activation Hessian $H^M_i$ could be obtained after concatenating all $(H^M_{i})^n$ as
\begin{align}
    H^M_i = [\diag((H^{M}_{i})^1);\ldots; \diag((H^{M}_i)^{{b H_{j}^o W_{j}^o}})]   
    \label{eq:pre_Hessian_cat}
\end{align}
the Hessian $\mathbf{H}^{oM}_{ij}$ for $\mathcal{W}^{oM}_{ij}$ can be obtained as:
\begin{align}
\mathbf{H}^{oM}_{ij} = \frac{1}{b H_{j}^{o} W_{j}^{o}}\sum_{n=1}^{b H_{j}^o W_{j}^o} (\mathbf{H}^{oM}_{ij})^n
\end{align}
It should be noted that as pre-activation Hessian is a recursive variable for convolutional layer and Hessian will be used for updating hyper-parameters which will be introduced later, both $H^{M}_{i}$ and $\mathbf{H}^{oM}_{ij}$ should be converted back to conv type before imparting to next layer with dimension $\R^{{b \times C_{i} \times H_{i} \times W_{i}}}$ and $\R^{C_{j}^{o} \times C_{i} \times m_{ij}^{o} \times k_{ij}^{o}}$.

As analyzed in Sec ~\ref{appendix:fchessian}, the Hessian calculation may cost a lot of time and resource. In order to address this problem, we propose the following approximate method:
\begin{align}
     {\rmH}^{oM}_{ij} = \mathbb{E}(M_{i})^2 \otimes \mathbb{E}(H^{oM}_{j})
    \label{eq:conv_Hessian_approximation}
\end{align}
\begin{align}
    H^M_i = \underbrace{{\left[{\hat{H}}^M_i;\ldots;{\hat{H}}^M_i 
    \right ]}}_{(b H_{\text{out}} W_{\text{out}})}, \ \
    \hat{H}_i^M = (B_i)^2 \circ (({\mathcal{W}^{oM}_{ij}})^2 \mathbb{E}(H^{oM}_j)^{\top}) + D_i
    \label{eq:conv_pre_hessian_approximation}
\end{align}
where \begin{math} B_i = \mathbb{E}(\sigma'(h_{i})), \ \
    D_i=\mathbb{E}(\sigma''(h_i) \circ \frac{{\partial}L}{{\partial}M_i}) \end{math}; 
$\mathbb{E}()$ returns a vector which stands for the mean value of input variable along feature map. An approximate pre-activation Hessian is calculated without decomposition of input variables, which saves more than $b H_{j}^o W_{j}^o$ times multiply-accumulate operations (MACs). 

\subsection{Compute the Hessian of Architecture Parameter}
\label{appendix:scalarhessian}
After we have the computation method for the Hessian of a Convolutional layer, we need to consider the Hessian of an architecture parameter. Now the output from node $i$ to $j$ under operation $o$ becomes $w_{ij}^o B_i$, where $w_{ij}^o$ is the architecture parameter and $B_i$ stands for the input vector.

Inspired by \cite{botev2017practical}, the Hessian for $w_{ij}^o$ could be computed recursively as $\mathbf{H}_{ij}^{o} = \mathbb{E}(\sum (B_i)^2 H_{j})$, 
where $H_{j}$ is supposed to be the known pre-activation Hessian for $B_{j}$ and $H_i$ is the pre-activation Hessian for $B_{i}$. 

\vspace{-0.7cm}
{\small
\begin{align}
    \label{eq:scalar_pre_Hessian}  
    H_i = \sum\limits_{o\in \mathcal{O}}({w_{ij}^{o}})^2 H_{j}.
\end{align}}Since $B_{i}$ and $H_{j}$ are independent of each other, the Hessian $\mathbf{H}_{ij}^o$ could also be calculated more efficiently:

\vspace{-0.7cm}
{\small
\begin{align}
\label{eq:scalar_pre_Hessian_approximation}
  \mathbf{H}_{ij}^o = (\mathbb{E}(|B_{i}|)^2 H_{j}
\end{align}}where $\mathbb{E}$ will return the mean.

\section{Detailed Settings of Experiments}
\label{app:expsetup}
\subsection{Architecture Search on CIFAR-10}

\paragraph{Data Pre-processing and Augmentation Techniques}
\label{datapr}
We employ the following techniques in our experiments: centrally padding the training images to $40\times 40$ and then randomly cropping them back to $32\times 32$; randomly flipping the training images horizontally; normalizing the training and validation images by subtracting the channel mean and dividing by the channel standard deviation.

\paragraph{Implementation Details of Operations}
\label{operations}
The operations include: 3 $\times$ 3 and 5 $\times$ 5 separable convolutions, 3 $\times$ 3 and 5 $\times$ 5 dilated convolutions, 3 $\times$ 3 max pooling, 3 $\times$ 3 average pooling, and skip connection. All operations are of stride one (excluded the ones adjacent to the input nodes in the reduction cell, which are of stride two) and the convolved feature maps are padded to preserve their spatial resolution. Convolutions are applied in the order of BN-ReLU-Conv and the depthwise separable convolution is always applied twice \citep{zoph2017learning, real2018regularized, liu2018progressive, liu2018darts}.

\paragraph{Detailed Training Settings}
\label{details}
The network parameters are optimized using momentum SGD, with initial learning rate $\eta_{\bm{\theta}} = 0.1$, momentum 0.9, and weight decay $1 \times 10^{-4}$. The batch size employed is 16 and the initial number of channels is 16.

\subsection{Architecture evaluation on CIFAR-10}

\paragraph{Additional Enhancement Techniques}

Following existing works \citep{zoph2017learning, liu2018progressive, pham2018efficient, real2018regularized, liu2018darts}, we employ the following additional enhancements: cutout \citep{devries2017improved}.

\subsection{Architecture transferability evaluation on CIFAR-10}

\paragraph{Detailed Training Settings}

The network is trained with batch size 128, SGD optimizer with weight decay $3\times 10^{-4}$, momentum 0.9 and initial learning rate 0.1, which is decayed using cosine annealing.

\subsection{Cells for ${\lambda}_w^o = 0.007$}
\begin{figure}[]
  \centering
    \includegraphics[scale = 0.4]{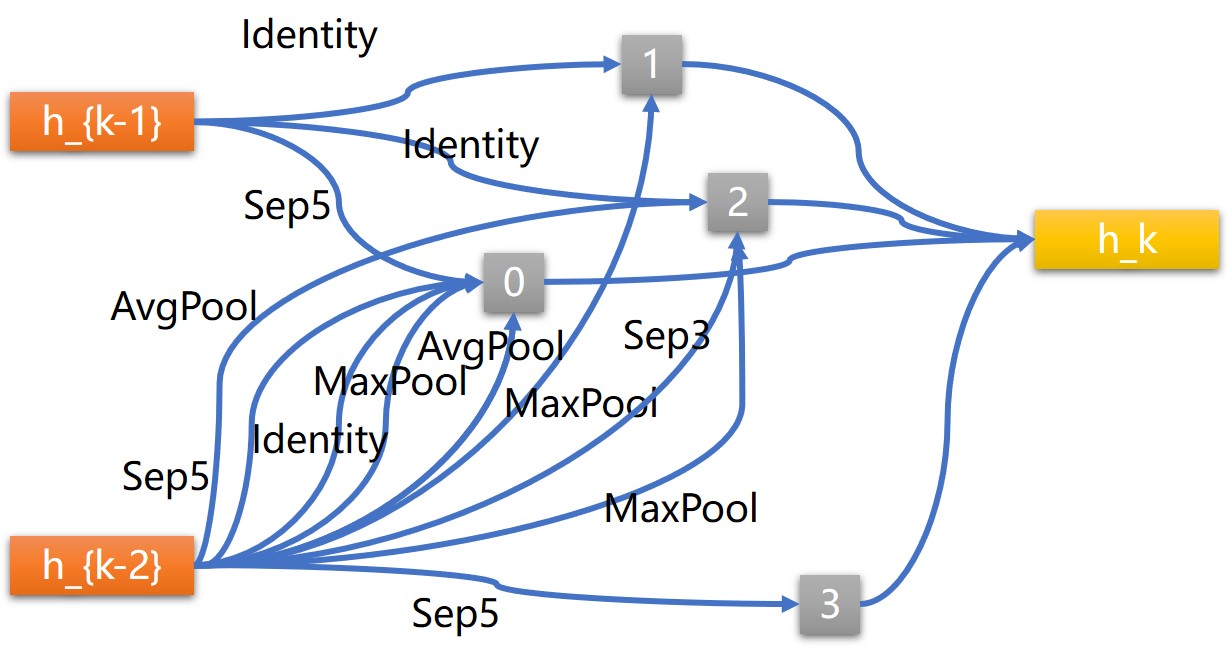}
  \caption{Normal cell found by BayesNAS with $\lambda_w^o = 0.007$.}
  \label{fig:cell}
\end{figure}

\begin{figure}[]
  \centering
    \includegraphics[scale = 0.4]{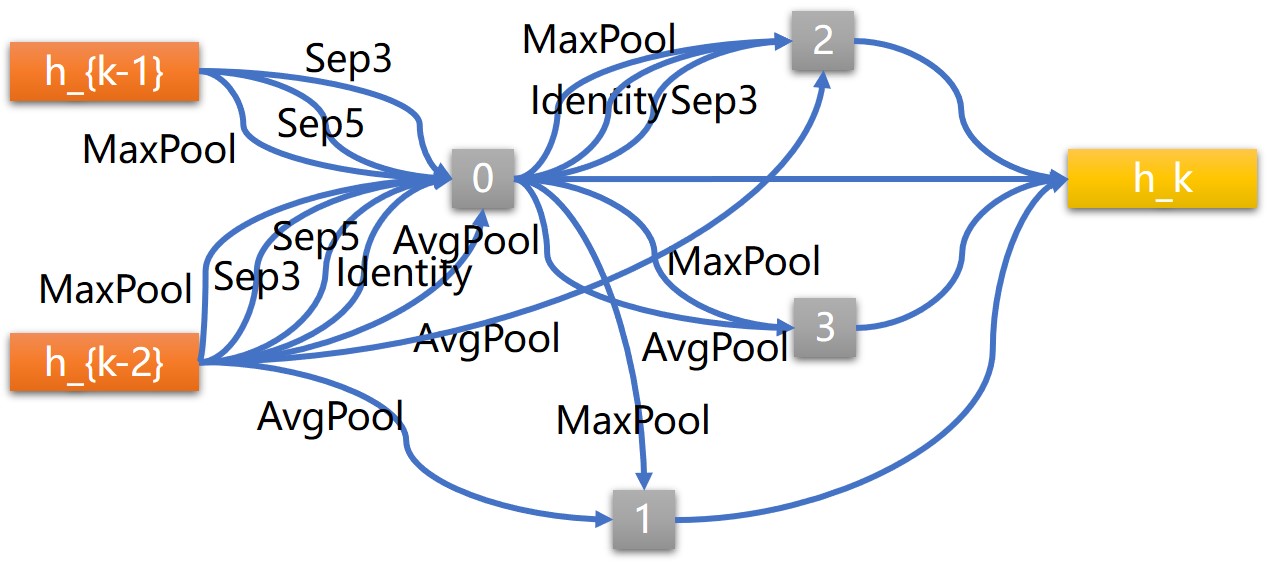}
  \caption{Reduction cell found by BayesNAS with $\lambda_w^o = 0.007$.}
  \label{fig:cell}
\end{figure}

\end{document}